\documentclass[final,onefignum,onetabnum]{siamonline190516}
\usepackage{amsmath,amssymb,bbm}
\usepackage{latexsym}
\usepackage{indentfirst}
\usepackage{graphicx}
\usepackage{placeins}
\usepackage{booktabs}
\usepackage{algorithm}
\usepackage{algorithmic}
\usepackage{multirow}
\usepackage{color}
\usepackage[bottom]{footmisc}

\usepackage{enumitem}
\setlist[enumerate]{leftmargin=.5in}
\setlist[itemize]{leftmargin=.5in}

\usepackage{hyperref}       
\usepackage{url}            
\usepackage{booktabs}       
\usepackage{amsfonts}       
\usepackage{nicefrac}       
\usepackage{microtype}      
\usepackage{xcolor}         

\usepackage{wrapfig, subcaption,caption}

\newsiamthm{claim}{Claim}
\newsiamremark{remark}{Remark}

\DeclareMathOperator{\argmin}{argmin}

\newcommand{\E}{\mathrm{E}}

\newcommand{\eps}{\epsilon}

\newcommand{\bbm}{\begin{bmatrix}}
\newcommand{\ebm}{\end{bmatrix}}

\newcommand{\R}{\mathrm{R}}

\newcommand{\one}{\mathbbm{1}}

\def\P{\mathbb{P}}
\def\E{\mathbb{E}}

\def\R{\mathbb{R}}

\def\d{\partial}

\DeclareMathOperator{\Tr}{Tr}
\headers{Combined resampling and reweighting}{Jing An and Lexing Ying}

\title{Combining resampling and reweighting for faithful stochastic optimization
\thanks{
Submitted to the editors DATE.
\funding{The work of L.Y. is partially supported by the U.S. Department of Energy, Office of Science, Office of Advanced
Scientific Computing Research, Scientific Discovery through Advanced Computing (SciDAC) program and also by the National Science Foundation under award DMS-1818449. J.A. was supported by Joe Oliger Fellowship from Stanford University.}}}

\author{Jing An\thanks{Max-Planck Institute for Mathematics in the Sciences, Inselstr. 22, 04103 Leipzig, Germany
  (\email{jing.an@mis.mpg.de}).}
\and Lexing Ying\thanks{Department of Mathematics, Stanford University, Stanford, CA 94305, USA
  (\email{lexing@stanford.edu})}}

\begin{document}
\maketitle
  
%




\begin{abstract}

Many machine learning and data science tasks require solving non-convex optimization problems. When
the loss function is a sum of multiple terms, a popular method is the stochastic gradient
descent. Viewed as a process for sampling the loss function landscape, the stochastic gradient
descent is known to prefer flat minima. Though this is desired for certain optimization 
problems such as in deep learning, it causes issues when the goal is to find the global minimum,
especially if the global minimum resides in a sharp valley.

Illustrated with a simple motivating example, we show that the fundamental reason is that the
difference in the Lipschitz constants of multiple terms in the loss function causes stochastic
gradient descent to experience different variances at different minima. In order to mitigate
this effect and perform faithful optimization, we propose a combined resampling-reweighting scheme
to balance the variance at local minima and extend to general loss functions. We explain from the numerical stability perspective how the proposed scheme is more likely to select the true global minimum, and the local convergence analysis perspective how it converges to a minimum faster when compared with the vanilla stochastic gradient descent. Experiments from robust
statistics and computational chemistry are provided to demonstrate the
theoretical findings.

\end{abstract}



\section{Introduction}
This paper is concerned with optimizing a non-convex smooth loss function. Identifying a global
minimum is known to be computationally hard, especially in the high-dimensional setting.  One
possible approach, originated from early work in statistical mechanics and Monte Carlo methods, is
to turn this into the task of sampling approximately the Gibbs distribution associated with the loss
function at a sufficiently low temperature. The rationale is that the samples from the Gibbs
distribution have a good chance of being near the global minimum.

In modern machine learning, the loss function often takes the form of an empirical sum of
individual terms from finitely many sampled data points. Due to the large size of the dataset,
efficient optimization methods such as the stochastic gradient descent (SGD) are commonly used. For
non-convex loss function, an increasingly more popular viewpoint is to consider SGD as a sampling
algorithm.



One important feature of stochastic gradient-type algorithms is that the noise drives SGD to escape
from sharp minima quickly and hence SGD prefers flat minima \cite{zhu2018anisotropic,
  xie2020diffusion, hoffer2017train,zhou2020towards, simsekli2019tail}. Such a bias towards flat
minima leads to better generalization properties for problems such as deep learning. However, when
the ultimate goal is to identify the global minimum and the landscape around the global minimum
happens to be sharper compared to the non-global local minima, this bias is often not desired as SGD
often misses the global minimum in a sharp valley.

For many data science and physical science problems, the ultimate goal is to find the global
minimum for a non-convex landscape, independent of whether it is sharp or flat. In data sciences, one
example is the handling of contaminated data, where a simple approach is to use non-convex loss
functions in robust statistics \cite{ma2020normalized, rousseeuw1984least,jain2017non,
  barron2019general}. However, as shown later in a motivating example, if we naively apply the vanilla
stochastic gradient over a dataset comprised of two subgroups, where one has much
larger features (more sensitive) compared to the other (less sensitive), the resulted optimal
parameter will be biased towards the less sensitive subgroup. This scenario is not rare in real
applications: for example, when researchers adjust a medicine's ingredients by evaluating the tested
group's responses, inherently different hormone levels in individuals can affect the faithfulness of
the evaluation. In physical sciences, examples of non-convex global minimization include finding the
ground state wave function in quantum many body problems \cite{kochkov2018variational}, geometry
optimization of the potential energy surface of a molecule in computational chemistry
\cite{leach2001molecular}, and etc. Though applying vanilla stochastic gradient can reduce
computational cost for these large scale problems, one also takes the risk of missing the global
minimum.

\subsection{Main contributions} Below we summarize the main contributions of this paper.
\begin{enumerate}
\item Starting from a motivating example, we identify the fundamental reason behind the
  selection bias is that the difference in the Lipschitz constants of multiple terms in the loss
  function causes stochastic gradient descent to experience different variances at different local
  minima.
\item To mitigate this selection bias, we propose a combined resampling-reweighting strategy for faithful
  minimum selection. We also derive stochastic differential equation (SDE) models to shed lights on how the proposed strategy balances variances in different regions. This proposed strategy
  also recovers the importance sampling SGD for faster training (for example, \cite{zhao2015stochastic,katharopoulos2018not}) from a different perspective.
  \item We provide a stability analysis and a local convergence rate for the importance sampling SGD in non-convex optimization problems. Furthermore, in the quantitative results we show how the combined resampling-reweighting strategy improves stability and convergence.
\item We show also empirically that the proposed strategy outperforms SGD with examples from robust statistics and computational chemistry.
\end{enumerate}


\subsection{Related work} 
Our proposed combined resampling-reweighting strategy can be viewed as a form of importance
sampling. This line of works can be traced back to the randomized Kaczmarz method
\cite{strohmer2009randomized} that selects rows with probability proportional to their squared
norms. Later, \cite{needell2014stochastic} connects the randomized Kaczmarz method with a SGD
algorithm with importance sampling. In convex optimization, many works
\cite{needell2014stochastic,zhao2015stochastic, canevet2016importance} show that the importance
sampling among stochastic gradients can improve the convergence speed. Since importance sampling
reduces the stochastic gradient's variance, this method and its variants can also accelerate the
neural networks training
\cite{alain2015variance,johnson2018training,loshchilov2015online,katharopoulos2018not}. However, it
has not been studied yet that how the importance sampling impacts the minimum selection in
learning non-convex problems.

Our approach to understanding the dynamics of the combined resampling-reweighting strategy is based on the numerical analysis for stochastic systems. We mention that the dynamical stability perspective has been used in for example \cite{wu2018sgd, ma2021sobolev}. Studying convergence rates for stochastic gradient algorithms for non-convex loss functions is also rapidly growing in recent years \cite{fehrman2020convergence, mertikopoulos2020almost, wojtowytsch2021stochastic}. On the other hand, taking the continuous-time limit and using SDEs to analyze stochastic algorithms have become popular
especially for stochastic non-convex problems. Using the developed stochastic analysis can lead to
numerous new insights of the non-convex optimization \cite{li2017stochastic,chaudhari2018stochastic,
  mandt2016variational, cheng2020stochastic}. Here we take the SDE approximation approach as it
gives us a clearer picture of the global minimum selection.




\section{Main idea from a motivating example}\label{sec:main_idea}

Given a dataset consisting of $n$ samples $\{x_i\}_{i=1}^n$, we consider an optimization problem
$$\min_\theta \frac{1}{n} \sum_{i=1}^n V(x_i,\theta).$$
The samples are  assumed to come from $m$
different subgroups, each representing a proportion $a_j\in(0,1)$ of the overall population, i.e.,
$\sum_{j=1}^m a_j = 1$. Assuming for simplicity that the loss term $V(x_i,\theta)$ only depends on the subgroup index
of $x_i$, i.e., $V(x_i,\theta) = V_j(\theta)$ if $x_i$ is from subgroup $j$, the optimization
problem can be simplified as 
\begin{equation}\label{opt:min}
  \theta^* = \argmin_{\theta} V(\theta), \quad V(\theta) \equiv \sum_{j=1}^m a_j V_j(\theta),
\end{equation}
in the large $n$ limit. If the terms $V_j(\theta)$ are non-convex loss functions, the overall loss
function $V(\theta)$ is in general non-convex as well. From the next example, we will show that
applying the vanilla SGD to solve (\ref{opt:min}) becomes problematic when $V_j$'s exhibit drastically
different Lipschitz constants.

\subsection{An illustrative example.}\label{sec:ex}

Consider the case of two subgroups with the following loss functions,
\[
V_1(\theta)=\begin{cases}
|\theta+1|-1, & \theta\le 0\\
\eps \theta, & \theta>0
\end{cases},
\quad
V_2(\theta)=\begin{cases}
-\eps \theta, & \theta\le 0\\
|K\theta-1|-1, & \theta>0
\end{cases}
\]
with $\eps\in(0,1)$ small, $K>1$, and $a_2>a_1$. The total loss function is $V(\theta) = a_1 V_1(\theta) + a_2
V_2(\theta)$ with local minima $\theta=-1$ and $\theta=1/K$. By construction, the loss function
has a sharp global minimum at $\theta=1/K$ and a flat local minimum at $\theta=-1$ (see for example 
Fig \ref{fig:small} (1)).
\footnote{It is necessary to have $O(\eps)$ terms in the loss function for SGD to work. Without the
  $O(\eps)$ terms, if the SGD starts in $(-\infty, 0)$ it will stay in this region because there is
  no drift from $V_2(x)$. Similarly, if the SGD starts in $(0,\infty)$, it will stay in this
  region. That means the result of SGD only depends on the initialization when $O(\eps)$ term is
  missing.}



The following lemma states that the optimization trajectory of the vanilla SGD is biased towards one
of the local minima, in the small learning rate $\eta$ limit.
\begin{lemma}\label{lem:1}
  When $\eta$ is sufficiently small, the equilibrium distribution of the vanilla SGD is given given
  by
\[
p(\theta) \sim
\begin{cases}
  \exp\left(-\frac{2}{a_1a_2\eta} V(\theta)\right) & \text{for}~\theta<0,\\
  \frac{1}{K^2} \exp\left(-\frac{2}{K^2a_1a_2\eta} V(\theta)\right) & \text{for}~\theta>0,
\end{cases}
\]
up to a normalizing constant.
\end{lemma}
The derivation follows from approximating the SGD updates by an SDE with a numerical error of order $O(\sqrt{\eta})$ in the weak sense. Because the loss function
considered is piece-wise linear, the approximate SDEs are of Langevin dynamics form with piecewise
constant noise coefficients. In particular, when the dynamics reaches equilibrium, the stationary
distribution of the stochastic process can be approximated by a Gibbs distribution. The detailed
computations of Lemma \ref{lem:1} is given in the Appendix \ref{sec:appendixA}.

From Lemma \ref{lem:1}, we can make the following surprising observation: {\em When $K\gg 1$, even
  though $\theta=1/K$ is the global minimum, the SGD trajectory spends most of the time near the
  non-global local minimum $\theta = -1$.} This is illustrated in Fig \ref{fig:small} (2) and the
fundamental reason is that the Lipschitz constant of the individual loss term affects the SGD
variance at individual local minimum, thus resulting in an undesired equilibrium distribution.

\begin{figure}[h!]
  \centering
  \begin{tabular}{ccc}
  \includegraphics[scale=0.22]{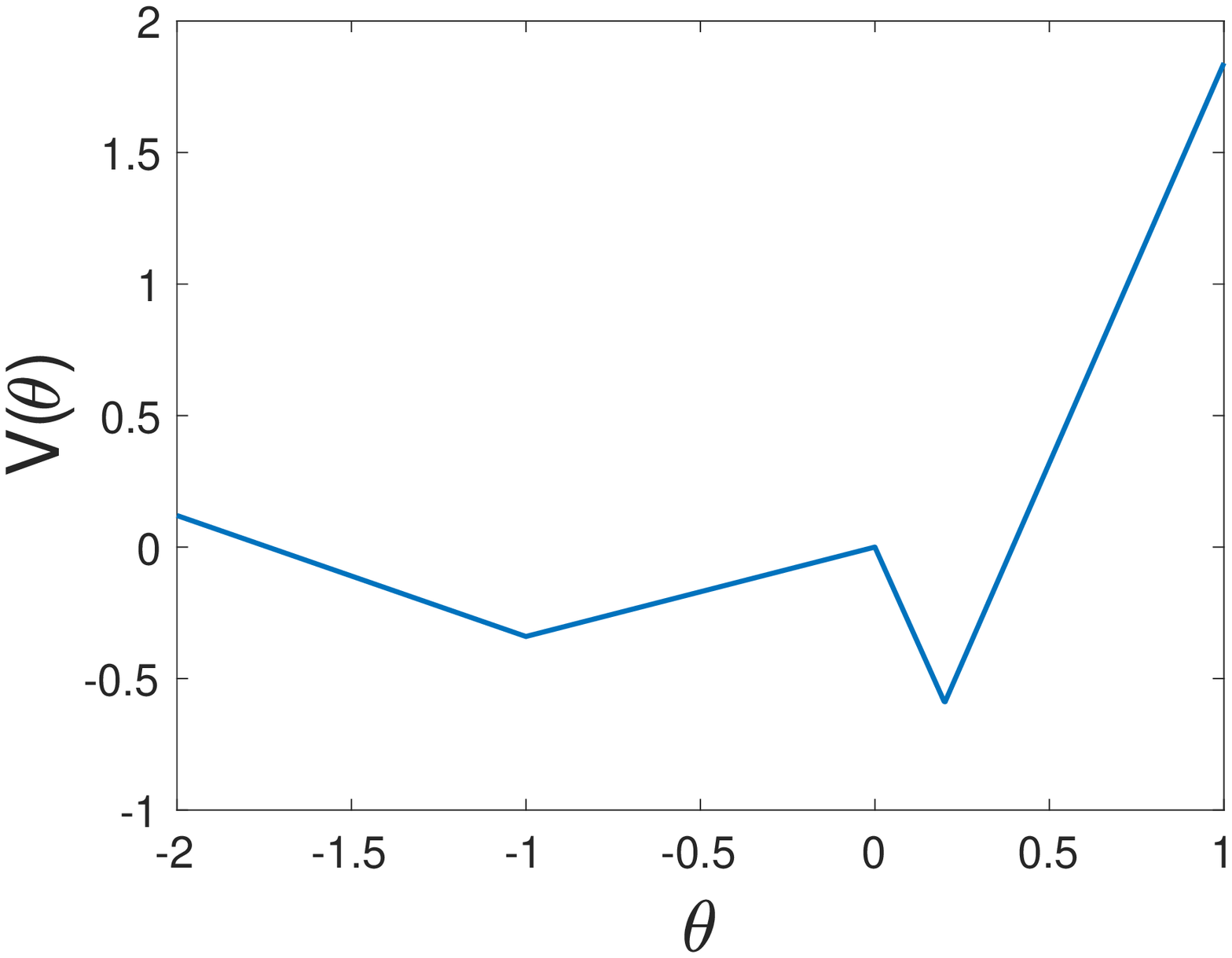}&
    \includegraphics[scale=0.22]{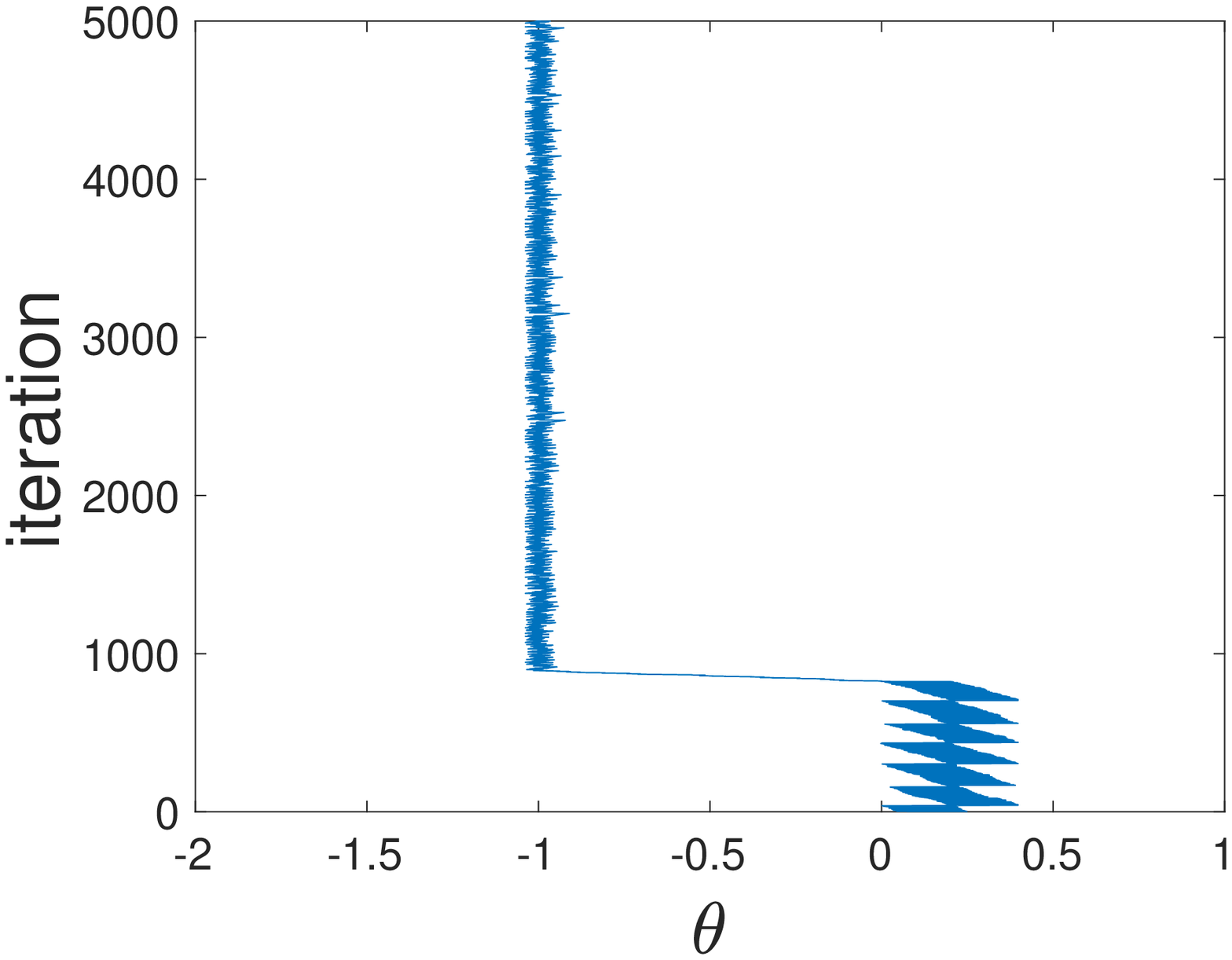} &
    \includegraphics[scale=0.22]{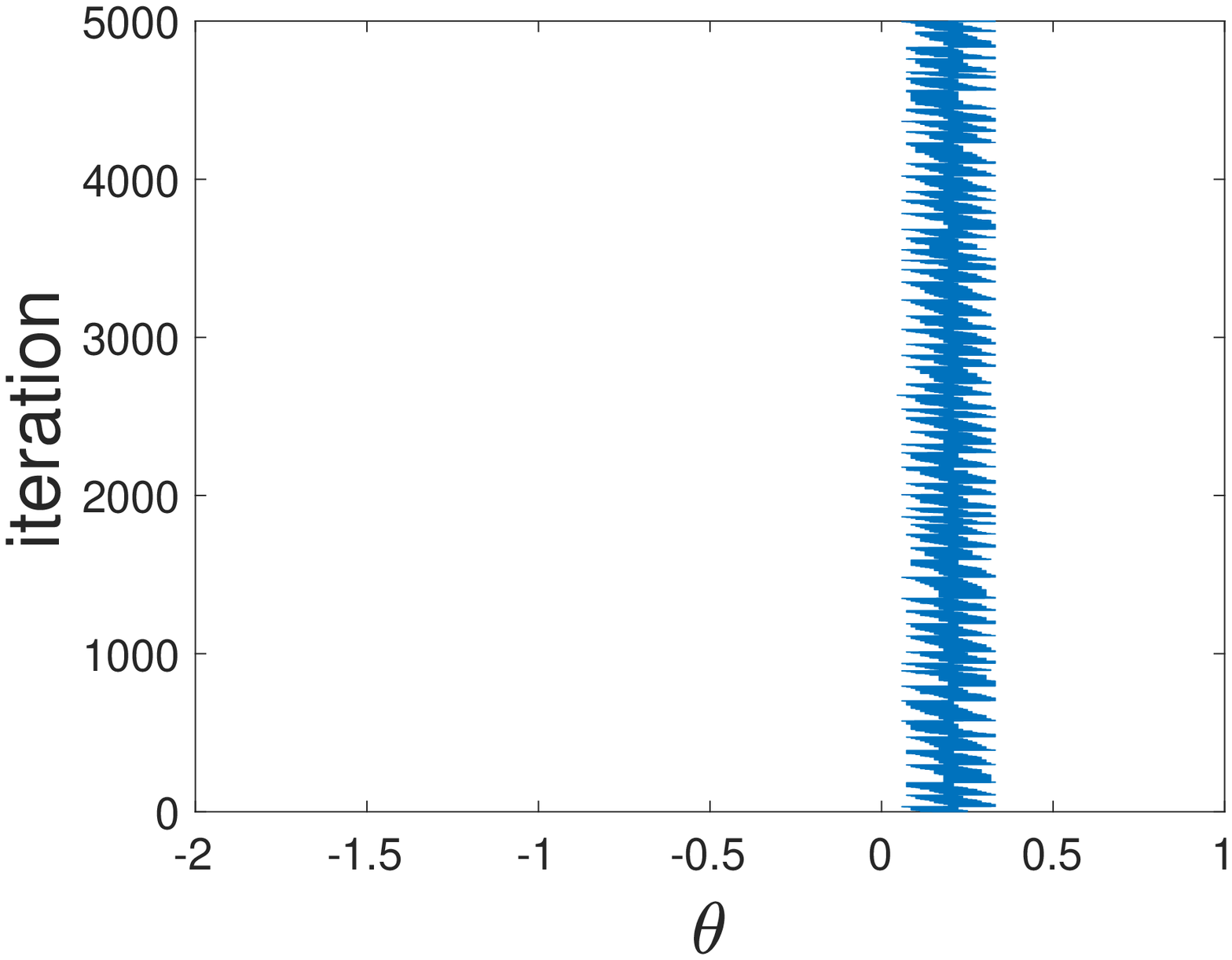}\\
     (1) Loss landscape &(2) SGD & (3) RR
  \end{tabular}
  \caption{We set $a_1 = 0.4$, $a_2 = 0.6$, $\eps = 0.1$, $K=5$ so that the global minimum is at
    $\theta=1/5$. For both the vanilla SGD and the resampling-reweighting (RR) scheme, we 
    start from $\theta_0 = 0.25$ and use a fixed step size
    $\eta=0.04$. We can see that the vanilla SGD jumps to the non-global local
    minimum $\theta=-1$ after several iterations while the RR method stays around the global
    minimum $\theta=1/K$ all the time. We include more comparisons with various learning rates in the Appendix \ref{sec:appendixC} to show that the RR scheme is stable for a wider range of $\eta$.}
  \label{fig:small}
\end{figure}

To fix this issue, we propose to {\em resample} two subgroups with proportion $f_1$ and $f_2$,
respectively (with $f_1, f_2>0, f_1+f_2=1$). In order to maintain the same overall loss function, we
also need to {\em reweight} each loss term with weights $w_1:=a_1/f_1, w_2:=a_2/f_2>0$.  The values
of $f_1,f_2,w_1$, and $w_2$ are to be determined depending on $a_1,a_2$ and $K$. After reweighting
and resampling, the loss function can be reformulated equivalently as
\begin{align}\label{eqn:reform}
  V(\theta) = f_1\cdot \Big( \frac{a_1}{f_1}V_1(\theta)\Big) + f_2\cdot \Big( \frac{a_2}{f_2}V_2(\theta)\Big).
\end{align}
In each iteration, a data point is sampled from the two subgroups following
proportion $f_1$ and $f_2$, and then either $\frac{a_1}{f_1}V_1(\theta)$ or $\frac{a_2}{f_2}V_2(\theta)$
is used for computing the stochastic gradient. In what follows, we refer to this approach as the
{\em resampling-reweighting} (RR) scheme.

Although under the expectation of the stochastic gradients in \eqref{eqn:reform} remains the same by
design, the variance experienced in different regions now can be balanced by the parameters
$f_1,f_2$. A direct computation shows that in the four regions $(-\infty, -1)$, $(-1,0)$, $(0,1/K)$,
and $(1/K, \infty)$:
\vspace{10px}
\begin{itemize}[leftmargin=*]
\item with probability $f_1$, the gradients are $-a_1/f_1$, $a_1/f_1$, $\eps a_1/f_1$, $\eps a_1/f_1$ respectively;
\item with probability $f_2$, the gradients are $-\eps a_2/f_2$, $-\eps a_2/f_2$, $-K a_2/f_2$, $K a_2/f_2$ respectively;
\item the variances of the gradients are equal to $\left(a_1 \sqrt{\frac{f_2}{f_1}} -\eps a_2
  \sqrt{\frac{f_1}{f_2}}\right)^2$, $\left(a_1 \sqrt{\frac{f_2}{f_1}} +\eps a_2
  \sqrt{\frac{f_1}{f_2}}\right)^2$, $\left(\eps a_1 \sqrt{\frac{f_2}{f_1}} +K a_2
  \sqrt{\frac{f_1}{f_2}}\right)^2$, $\left(\eps a_1 \sqrt{\frac{f_2}{f_1}} -K a_2
  \sqrt{\frac{f_1}{f_2}}\right)^2$, respectively.
\end{itemize}

\vspace{10px}
The dynamics of RR becomes straightforward if we view it as a numerical approximation of SDEs: with
a sufficiently small step size $\eta>0$, taking $\eps \to 0$, it is given by
\begin{equation*}
  \begin{aligned}
    &d\Theta_t = -V'(\Theta_t) dt + a_1\sqrt{f_2/f_1}\sqrt{\eta} dW_t,~~   &\text{in the region}~ \theta<0,\\
    &d\Theta_t = -V'(\Theta_t) dt + K a_2\sqrt{f_1/f_2} \sqrt{\eta}dW_t,~~ &\text{in the region}~ \theta>0.
  \end{aligned}
\end{equation*}
In order to balance the variance at the two local minima $\theta=-1$ and $\theta=1/K$, we impose
\begin{align*}
  a_1 \sqrt{f_2/f_1} = K a_2 \sqrt{f_1/f_2} ~~\Longrightarrow~~~
  f_1=\frac{a_1}{a_1+K a_2},~~f_2 = 1-f_1= \frac{K a_2}{a_1+K a_2}.
\end{align*}
As a result, the assigned weights for the two subgroups are
\begin{align*}
  w_1 := \frac{a_1}{f_1} = \frac{a_1+Ka_2}{1}, \quad w_2 := \frac{a_2}{f_2} = \frac{a_1+Ka_2}{K}.
\end{align*}
The above computation
suggests that in order to fix the selection bias, each subgroup should be reweighted by the
reciprocal of its Lipschitz constant. We then undersample from subgroup 1 and oversample
from subgroup 2 so that their sample size ratio approaches $Ka_2/a_1$. In numerical tests,
comparing Fig \ref{fig:small} (2) and (3), we can see that the RR scheme adjusts the dynamics of the
stochastic optimization trajectory to stay around the global minimum.

\begin{remark}
The motivating example considers different slopes $1$ and $K$, which can reflect the feature magnitude disparities in data science. Suppose a dataset containing $n$ samples $\{x_i, y_i\}_{i=1}^n$, $x_i\in\R^d, y_i\in\R$, and $a_1$-proportion of the data features has magnitude $\Vert x_i\Vert_2\sim K$, while the rest has magnitude $\Vert x_i\Vert_2\sim 1$, then an optimization problem of the form
\begin{align}
    \min_{\theta\in\R^d} \frac{1}{2n}\sum_{i=1}^n (f(\theta\cdot x_i)-y_i)^2
\end{align}
can be thought of a more complex model of our motivating example.
\end{remark}

\section{Analysis of the general case}\label{sec:method}
In this section, we consider the general empirical loss for $\theta\in\R^d, d\geq 1$.
\begin{align}\label{model3}
  L(\theta) = \frac{1}{n}\sum_{i=1}^n l_i(\theta).
\end{align}

\subsection{The general RR scheme}
Motivated by the illustrative example in Section \ref{sec:main_idea}, we propose the following resampling-reweighting scheme 
in $\R^d,~d\geq 1$: at each $k$-th iteration with current parameter $\theta_k$, for each 
term $i=1,2,\cdots n$,
\begin{enumerate}
\item reweight the $i$-th term $l_i(\theta_k)$ with a weight proportional to $1/\Vert\nabla
  l_i(\theta_k)\Vert_2$,
\item set the resampling probability for the (reweighted) $i$-th term to $\Vert \nabla
  l_i(\theta_k)\Vert_2$.
\end{enumerate}


The resampling proportions and weights are decided by comparing the RR scheme with SGD in
the continuous-time limit. More specifically, since the proposed RR scheme is designed for variance-balancing, in the limit it should share the same drift with the vanilla SGD but experience more balanced noises across different local minima. We lay out computation details in the next sub-section.

\subsubsection{RR scheme derivations} 
Recall for the vanilla SGD, the update rule with time step size $\eta>0$ is given by
\begin{align}\label{SGD}
  \theta_{k+1} = \theta_k - \eta \nabla l_j(\theta_k),
\end{align}
where the index $j$ is chosen from $1$ to $n$ with uniform probability $1/n$. Note that
\begin{align}\label{mean1}
  m(\theta_k) = \E_{p_1}[\nabla l_i(\theta_k)] = \frac{1}{n}\sum_{i=1}^n  \nabla l_i(\theta_k) = \nabla L(\theta_k),
\end{align}
and we can rewrite (\ref{SGD}) as
\begin{align*}
  \theta_{k+1} = \theta_k - \eta m(\theta_k) + \sqrt{\eta}V^1(\theta_k),\quad
  \text{with}~~V^1(\theta_k) = \sqrt{\eta}\left(m(\theta_k)- \nabla l_j(\theta_k)\right).
\end{align*}
By taking the simplifying assumption that the gradient noise is Gaussian\footnote{We are aware that
  this approximation might not be valid in many situations. However, this assumption streamlines the
  SDE analysis.}, the dynamics can be approximated by
\begin{align}\label{sme1}
  d\Theta_t = -m(\Theta_t)dt + \sqrt{\eta} \sigma^1(\Theta_t)dW_t,
\end{align}
where $\Sigma^1(\Theta_t ):= \sigma^1(\Theta_t )\sigma^1(\Theta_t )^{\top}$ is given by
$\Sigma^1(\Theta_t ) = \frac{1}{n} \sum_{i=1}^n \nabla l_i(\Theta_t ) \nabla
l_i(\Theta_t)^{\top}-m(\Theta_t ) ^{\otimes 2}$.

On the other hand, as indicated in the illustrative example, the RR scheme with the same time step size $\eta>0$ should be given by 
\begin{align}\label{IpSamp}
  \theta_{k+1} = \theta_k - \eta \frac{C(\theta_k) }{ \Vert \nabla l_j(\theta_k)\Vert_2}\nabla l_j(\theta_k).
\end{align}
Here we set $C(\theta_k) = \frac{1}{n}\sum_{i=1}^n \Vert\nabla l_i(\theta_k)\Vert_2$, and the index $j$ is chosen from $1$ to $n$ with probability $\Vert\nabla
l_i(\theta_k)\Vert_2/Z(\theta_k)$, with the normalizing factor $Z(\theta_k) = \sum_{i=1}^n
\Vert\nabla l_i(\theta_k)\Vert_2$, so that the mean for this approach matches with the one for SGD
(\ref{mean1}),
\begin{align*}
  \sum_{i=1}^n \frac{C(\theta_k) }{ \Vert \nabla l_i(\theta_k)\Vert_2}\nabla l_i(\theta_k)
  \frac{\Vert\nabla l_i(\theta_k)\Vert_2}{Z(\theta_k)}= \frac{1}{n}\sum_{i=1}^n \nabla l_i(\theta_k)
  = m(\theta_k).
\end{align*}
Now we rewrite (\ref{IpSamp}) in the form
\begin{align*}
  \theta_{k+1} = \theta_k - \eta m(\theta_k) + \sqrt{\eta}V^2(\theta_k), \quad \text{with}~~
  V^2(\theta_k) = \sqrt{\eta}\left(m(\theta_k)- \frac{C(\theta_k) }{ \Vert\nabla
    l_j(\theta_k)\Vert_2}\nabla l_j(\theta_k)\right).
\end{align*}
Thus, the resulted dynamics can be approximated as
\begin{align}\label{sme2}
    d\Theta_t = -m(\Theta_t)dt + \sqrt{\eta} \sigma^2(\Theta_t)dW_t,
\end{align}
where $\Sigma^2(\Theta_t )= \sigma^2(\Theta_t )\sigma^2(\Theta_t )^{\top}$ is given by
$\Sigma^2(\Theta_t ) = \frac{Z(\Theta_t)}{n^2}\sum_{i=1}^n \frac{\nabla l_i(\Theta_t ) \nabla
  l_i(\Theta_t ) ^{\top}}{\Vert\nabla l_i(\Theta_t)\Vert_2} - m(\Theta_t )^{\otimes 2}$.

Note that (\ref{sme1}) and (\ref{sme2}) are derived in the same time scale and have the same drift, so that the comparison of those two reduces to comparing covariance matrices $\Sigma^1, \Sigma^2$.

\begin{remark}
The RR scheme (\ref{IpSamp}) from the derivation turns out to be similar to
\cite{zhao2015stochastic} for faster convergence of regularized convex minimization problems,
\cite{katharopoulos2018not} for training deep learning with importance sampling, and even
\cite{strohmer2009randomized} for solving linear systems of equations earlier on.

In terms of computational complexity, recomputing the weights and sampling proportions in each iteration is obviously expensive. In practice, as discussed in \cite{needell2014stochastic}, one can use a rejection sampling scheme to select stochastic gradients from a weighted distribution. Algorithm 1 in \cite{katharopoulos2018not} presents a practical guidance on how to speedup the importance sampling SGD/RR scheme in deep learning, by only updating parameters by (\ref{IpSamp}) when the variance of gradients can be reduced.
\end{remark}

\subsection{Stability analysis}
Let us step back to discrete time stochastic algorithms, and compare SGD and RR scheme from the numerical analysis perspective. In fact, especially for stiff problems, the RR scheme allows a wider range of step-sizes to keep stochastic linear stability compared to SGD. We consider minimizing the empirical loss function (\ref{model3}) of the form
\begin{equation}
    L(\theta) = \frac{1}{n}\sum_{i=1}^n l_i(\theta)=\frac{1}{2n}\sum_{i=1}^n \left(f(x_i,\theta)-y_i\right)^2,
\end{equation}
where $f$ is the learning model and $\{(x_i,y_i)\}_{i=1}^n, x_i\in\R^d, y_i\in\R$ are i.i.d sampled data points. We assume that for the model $f$, there is an interpolation solution $\theta^*$ such that
\begin{align*}
    y_i = f(x_i,\theta^*),\quad \forall~1\leq i\leq n.
\end{align*}
By definition, a stationary point $\tilde \theta$ is \textit{stochastically stable} if there exists a uniform constant $0 < C \leq 1$ such that 
$\E[\lVert \theta_k - \tilde \theta \rVert_2^2] \leq C \lVert \theta_0 -
\tilde \theta \rVert_2^2, k\geq 1$, where $\theta_k$ is the $k$-th iterate of the stochastic algorithm. Suppose $\{\theta_k\}$ are sufficiently close to the interpolation solution $\theta^*$, then SGD iteration can be written as
\begin{equation}\label{stability1}
\begin{aligned}
    \theta_{k+1}&=\theta_k -\eta (f(x_i,\theta_k)-y_i)\nabla_{\theta} f(x_i,\theta_k)\\
    &\approx \theta_k -\eta \nabla_{\theta} f(x_i, \theta^*)\nabla_{\theta} f(x_i, \theta^*)^{\top} (\theta_k-\theta^*),
\end{aligned}    
\end{equation}
where we take a Taylor expansion approximation $f(x_i,\theta_k)-y_i =f(x_i,\theta_k)-f(x_i,\theta^*) \approx \nabla_{\theta}f(x_i, \theta^*)^{\top} (\theta_k-\theta^*)$ and also take $\nabla_{\theta}f(x_i, \theta_k)\approx \nabla_{\theta}f(x_i, \theta^*)$ approximately. Let us denote
\begin{equation}
    H_i := \nabla_{\theta} f(x_i, \theta^*)\nabla_{\theta} f(x_i, \theta^*)^{\top}\quad \text{and }\quad H:= \frac{1}{n}\sum_{i=1}^n H_i,
\end{equation}
then we have the following stability conditions for SGD and RR scheme.
\begin{lemma}\label{lem:stability}
Suppose that the starting point $\theta_0$ is close to the interpolation solution $\theta^*$, that is, there exists small $\eps>0$ such that $\lVert \theta_0-\theta^*\rVert_2\leq \eps$ , and $\nabla_{\theta} f(x_i, \theta)$ is bounded and continuous around $\theta^*$ for all $1\leq i\leq n$. Then the condition for the SGD to be stochastically stable around  $\theta^*$ is
\begin{equation}\label{cond1}
    \lambda_{\text{max}}\left\{(I-\eta H)^2+ \eta^2\left(\frac{1}{n}\sum_{i=1}^n H_i^2 - H^2\right)+O(\eps)\right\}\leq 1,
\end{equation}
and the condition for the RR scheme to be stochastically stable around  $\theta^*$ is, for each $k\geq 0$,
\begin{equation}\label{cond2}
    \lambda_{\text{max}}\left\{(I-\eta H)^2+ \eta^2\left(\frac{1}{n}\sum_{i=1}^n\tilde G_{i,k} H_i^2- H^2\right)+O(\eps)\right\}\leq 1,
\end{equation}
where $\tilde G_{i,k}:=\frac{\frac{1}{n}\sum_{j=1}^n\lVert \nabla_{\theta} f(x_j, \theta^*)\nabla_{\theta} f(x_j, \theta^*)^{\top}(\theta_k-\theta^*)\rVert_2}{\lVert \nabla_{\theta} f(x_i, \theta^*)\nabla_{\theta} f(x_i, \theta^*)^{\top}(\theta_k-\theta^*)\rVert_2}$.
\end{lemma}
\begin{proof}
The first part (\ref{cond1}) closely follows from \cite{wu2018sgd}. Indeed, the iteration (\ref{stability1}) can be rewritten more precisely as
\begin{equation}
    \theta_{k+1}-\theta^* = \left(I - \eta H_i+O(\eps)\right)(\theta_k-\theta^*).
\end{equation}
Conditioned on $\{\Vert\theta_k-\theta^*\Vert_2\leq \eps\}$, we have the expectation on the second moment as
\begin{equation}\label{compute_cond1}
    \begin{aligned}
        \E_{\mathcal{D}}[\lVert\theta_{k+1}-\theta^*\rVert_2^2] &= \E_{\mathcal{D}}\left[(\theta_k-\theta^*)^{\top}\left(I - \eta H_i+O(\eps)\right)^2(\theta_k-\theta^*)\right]\\
        & = (\theta_k-\theta^*)^{\top}\left((I-\eta H)^2+ \eta^2\left(\frac{1}{n}\sum_{i=1}^n H_i^2 - H^2\right)+O(\eps)\right)(\theta_k-\theta^*).
    \end{aligned}
\end{equation}
The step (\ref{compute_cond1}) can be iterated down to $\theta_0$, and the expectation is taken over uniform distribution $\mathcal{D}$. To ensure that the stochastic stability condition is satisfied, we then need (\ref{cond1}). On the other hand, the iteration for the RR scheme reads
\begin{equation}\label{iterRR}
\begin{aligned}
    \theta_{k+1} &= \theta^{k} -\eta G(x_i,\theta_k)(f(x_i,\theta_k)-y_i)\nabla_{\theta} f(x_i,\theta_k)\\
    &=\theta_k -\eta (\tilde G_{i,k}+O(\eps))\left( \nabla_{\theta} f(x_i, \theta^*)\nabla_{\theta} f(x_i, \theta^*)^{\top}+O(\eps)\right) (\theta_k-\theta^*),
\end{aligned}   
\end{equation}
where each data sample $x_i\sim \mathcal{D}^{(w)}$ is selected with probability
\begin{align*}
    p(x_i,\theta_k) = \frac{\lVert (f(x_i,\theta_k)-y_i)\nabla_{\theta} f(x_i,\theta_k) \rVert_2}{\sum_{j=1}^n \lVert (f(x_j,\theta_k)-y_j)\nabla_{\theta} f(x_j,\theta_k)\rVert_2} = \frac{\lVert \nabla_{\theta} f(x_i, \theta^*)\nabla_{\theta} f(x_i, \theta^*)^{\top}(\theta_k-\theta^*)\rVert_2}{\sum_{j=1}^n\lVert \nabla_{\theta} f(x_j, \theta^*)\nabla_{\theta} f(x_j, \theta^*)^{\top}(\theta_k-\theta^*)\rVert_2}+O(\eps)
\end{align*}
given $\{\Vert\theta_k-\theta^*\Vert_2\leq \eps\}$, and $G(x_i,\theta_k) = C(\theta_k)/\lVert \nabla l_i(\theta_k\rVert_2$, which can be rewritten here as
\begin{align*}
    G(x_i,\theta_k) &= \frac{\frac{1}{n}\sum_{j=1}^n \lVert(f(x_j,\theta_k)-y_j)\nabla_{\theta} f(x_i,\theta_k) \rVert_2}{\lVert(f(x_i,\theta_k)-y_i)\nabla_{\theta} f(x_i,\theta_k) \rVert_2}\\
    &= \frac{\frac{1}{n}\sum_{j=1}^n\lVert \nabla_{\theta} f(x_j, \theta^*)\nabla_{\theta} f(x_j, \theta^*)^{\top}(\theta_k-\theta^*)\rVert_2}{\lVert \nabla_{\theta} f(x_i, \theta^*)\nabla_{\theta} f(x_i, \theta^*)^{\top}(\theta_k-\theta^*)\rVert_2}+O(\eps) = \tilde G_{i,k}+O(\eps).
\end{align*}
Therefore, the RR iteration (\ref{iterRR}) can be rewritten as 
\begin{equation}
    \theta_{k+1}-\theta^*= (I-\eta \tilde G_{i,k} H_i + O(\eps))(\theta_k-\theta^*).
\end{equation}
Conditioned on $\{\Vert\theta_k-\theta^*\Vert_2\leq \eps\}$, we compute the expectation on the second moment
\begin{equation}\label{compute_cond2}
    \begin{aligned}
        \E_{\mathcal{D}^{(w)}}[&\lVert\theta_{k+1}-\theta^*\rVert_2^2] = \E_{\mathcal{D}^{(w)}}\left[(\theta_k-\theta^*)^{\top}\left(I - \eta \tilde G_i H_i+O(\eps)\right)^2(\theta_k-\theta^*)\right]\\
        & = (\theta_k-\theta^*)^{\top}\left((I-\eta H)^2+ \eta^2\left(\frac{1}{n}\sum_{i=1}^n\tilde G_{i,k} H_i^2- H^2\right)+O(\eps)\right)(\theta_k-\theta^*).
    \end{aligned}
\end{equation}
Similar to what we analyzed for SGD, the stability condition for RR scheme is deduced to be (\ref{cond2}).
\end{proof}
\begin{remark}
When $d=1$, all matrices in Lemma \ref{lem:stability} are scalar. It is straightforward to see that, for the second component in RR scheme, by Cauchy-Schwarz inequality, we get
\begin{align*}
   \frac{1}{n}\sum_{i=1}^n\tilde G_{i,k} H_i^2- H^2 &= \frac{1}{n}\sum_{i=1}^n\frac{\frac{1}{n}\sum_{j=1}^n |\nabla_{\theta}f(x_j,\theta^*)|^2}{|\nabla_{\theta}f(x_i,\theta^*)|^2}|\nabla_{\theta}f(x_i,\theta^*)|^4-H^2\\
   &= \frac{1}{n^2} \left(\sum_{i=1}^n |\nabla_{\theta}f(x_i,\theta^*)|^2\right)^2 \leq \frac{1}{n}\sum_{i=1}^n|\nabla_{\theta}f(x_i,\theta^*)|^4-H^2 = \frac{1}{n}\sum_{i=1}^n H_i^2-H^2.
\end{align*}
When the model $f$ has drastically different gradient magnitude in different data samples, compared to SGD, the RR scheme fundamentally acts as gradient magnitude averaging to allow a broader range of learning rates $\eta$ for stochastic stability. The RR scheme exhibits its particular strength of maintaining stability for stiff problems.
\end{remark}
\subsection{Local convergence analysis} 
In the non-convex scenario, it is relatively easier to obtain convergence results locally. In order to make a direct comparison with SGD in terms of local convergence rates, we adapt Theorem 4 and its proof strategy from \cite{mertikopoulos2020almost} to establish convergence analysis of the RR scheme. We use the same setup as in \cite{mertikopoulos2020almost}, which takes a step-size schedule of the form $\eta_k = \frac{\gamma}{(k+m)^p}$ for some $p\in(1/2, 1]$ with sufficiently large $\gamma, m>0$.  The key ingredient is the following decomposition for stochastic gradients
\begin{align}
    V_{k,j}:=\nabla l_j(\theta_k) &= \nabla L(\theta_k) + Z_j(\theta_k), \quad \quad  j\sim\mathcal{D}\quad \text{for SGD},\\
    \tilde V_{k,j}:=G_j(\theta_k)\nabla l_j(\theta_k) &= \nabla L(\theta_k) + \tilde Z_j(\theta_k), \quad \quad  j\sim\mathcal{D}^{(w)}\quad \text{for RR},
\end{align}
where $\mathcal{D}$ denotes the uniform distribution over samples, and $\mathcal{D}^{(w)}$ is the weighted distribution used in the RR scheme. $G_j(\theta_k) = \frac{1}{n}\sum_{i=1}^n \Vert\nabla l_i(\theta_k)\Vert_2/\Vert\nabla l_j(\theta_k)\Vert_2$ as before. Note that for both cases $\E_{ \mathcal{D}}[Z_{j}(\theta_k)] = \E_{ \mathcal{D}^{(w)}}[\tilde Z_{j}(\theta_k)]=0$, and we further assume that given a neighborhood $\mathcal{U}$ of $\theta_*$, for each $1\leq j\leq n$, there exists $\sigma_j>0$ such that
\begin{align}
    \sup_{\theta\in\mathcal{U}}\Vert\nabla l_j(\theta)\Vert_2^2\leq \sigma_j^2.
\end{align}
\begin{theorem}
Suppose in a convex compact neighborhood of $\theta_*$, there exists $\alpha>0$ so that $\nabla^2 L(\theta)\succcurlyeq \alpha I$. Given the assumptions and the step-size schedule introduced above, for a fixed $\delta\in(0,1)$, there exist neighborhoods $\mathcal{U}, \mathcal{U}_1$ containing $\theta_*$ such that
\begin{align}
    \P(E_{\infty}=\{\theta_k\in\mathcal{U}~\text{for all}~k\geq 1\}|\theta_1\in\mathcal{U}_1)\geq 1-\delta.
\end{align}
Furthermore, we have the local convergence rate for SGD
\begin{align}
    \E[\Vert\theta_k-\theta_*\Vert_2^2|E_{\infty}]&\leq \frac{\gamma^2}{(1-\delta)(2\alpha \gamma-1)k}\sup_j \sigma_j^2+o(1/k),\quad \text{if}~p=1,\\
    \E[\Vert\theta_k-\theta_*\Vert_2^2|E_{\infty}]&\leq \frac{\gamma}{2\alpha(1-\delta)k^p}\sup_j \sigma_j^2+o(1/k^p),\quad \text{if}~p<1.
\end{align}
And for the RR scheme, we have
\begin{align}
    \E[\Vert\theta_k-\theta_*\Vert_2^2|E_{\infty}]&\leq \frac{\gamma^2}{(1-\delta)(2\alpha \gamma-1)k}\left(\frac{1}{n}\sum_{i=1}^n \sigma_i^2\right)+o(1/k),\quad \text{if}~p=1,\\
    \E[\Vert\theta_k-\theta_*\Vert_2^2|E_{\infty}]&\leq \frac{\gamma}{2\alpha(1-\delta)k^p}\left(\frac{1}{n}\sum_{i=1}^n \sigma_i^2\right)+o(1/k^p),\quad \text{if}~p<1.
\end{align}
When $p=1$, we choose $\gamma$ large so that $2\alpha\gamma>1$.
\end{theorem}
\begin{proof} First of all, due to the local convexity of $L(\theta)$ near $\theta_*$, we have
\begin{align*}
   \langle \theta-\theta_*, \nabla L(\theta)\rangle \geq  \alpha \Vert \theta-\theta_*\Vert_2^2,
\end{align*}
for all $\theta\in\mathcal{K}$, where $\mathcal{K}$ is a convex compact neighborhood of $\theta_*$. We have the stochastic gradient updates as $\theta_{k+1} = \theta^{k}-\eta_k V_{k,j}$. Let $D_k := \Vert \theta_k-\theta_*\Vert_2^2$, then
\begin{equation}\label{aug261}
\begin{aligned}
    D_{k+1} &= \Vert \theta_k-\theta_*-\eta_k V_{k,j}\Vert_2^2\\
    &= \Vert \theta_k-\theta_*\Vert_2^2 - 2\eta_k\langle \theta_k-\theta_*, V_{k,j} \rangle + \eta_k^2 \Vert V_{k,j}\Vert_2^2\\
    & = \Vert \theta_k-\theta_*\Vert_2^2 - 2\eta_k\langle \theta_k-\theta_*, \nabla L(\theta_k) \rangle - 2\eta_k\langle \theta_k-\theta_*, Z_j(\theta_k) \rangle+ \eta_k^2 \Vert V_{k,j}\Vert_2^2\\
    &\leq (1-2\alpha \eta_k )D_k + 2\eta_k \xi_{k,j} + \eta_k^2 \Vert V_{k,j}\Vert_2^2,
\end{aligned}
\end{equation}
where $\xi_{k,j}=-\langle \theta_k-\theta_*, Z_j(\theta_k) \rangle $ and $\E[\xi_{k,j}|\mathcal{F}_k] = 0$. Same computations hold for RR scheme as well by replacing notations accordingly. The key idea of the proof is to control the error aggregation in $2\eta_{k} \xi_{k,j} + \eta_k^2 \Vert V_{k,j}\Vert_2^2$. We will only sketch the main steps and highlight where the difference between SGD and RR emerges, since the proof details can be found in \cite{mertikopoulos2020almost}. From now on, unless it is necessary, we omit the second subscript in $\xi, V$ for simplicity as they change in each iteration. For the error terms, one can define 
\begin{align}
    M_k =2 \sum_{l=1}^k \eta_l \xi_l \quad \text{and}\quad S_k = \sum_{l=1}^k \eta_l^2 \Vert V_l\Vert_2^2.
\end{align}
Define the cumulative mean square error $R_k = M_k^2+S_k$, then we have
\begin{equation}
    \begin{aligned}
        R_k &= (M_{k-1}+2\eta_k\xi_k)^2 + S_{k-1}+\eta_k^2\Vert V_k\Vert_2^2\\
        &=R_{k-1}+4M_{k-1}\eta_k\xi_k+4\eta_k^2\xi_k^2+\eta_k^2\Vert V_k\Vert_2^2.
    \end{aligned}
\end{equation}
It is easy to check that $R_k$ is sub-martingale, $\E[R_k|\mathcal{F}_k]\geq R_{k-1}$. The proof uses a finer condition by introducing the following events, let $\mathcal{U}$ be a neighborhood of $\theta^*$ and $\eps>0$,
\begin{align}
    E_k =\{\theta_l\in\mathcal{U} \text{ for all } 1\leq l\leq k\} \quad \text{and}\quad H_k = \{R_l\leq \eps \text{ for all } 1\leq l\leq k\}.
\end{align}
The property analysis of $E_k$ and $H_k$ is the same as in Lemma D.2 in \cite{mertikopoulos2020almost}, except for (D.19). Notice that (D.26) can be estimated as
\begin{align}
    \E_{\mathcal{D}}[\Vert V_{k,j}\Vert_2^2\one_{H_{k-1}}]&\leq \sup_j \Vert \nabla l_j(\theta_k)\Vert_2^2,\quad \text{for SGD}, \label{bound_1}\\
    \E_{\mathcal{D}^{(w)}}[\Vert \tilde V_{k,j}\Vert_2^2\one_{H_{k-1}}]&\leq \E_{\mathcal{D}^{(w)}}[G_j(\theta_k)^2\Vert \nabla l_j(\theta_k)\Vert_2^2] = \frac{1}{n^2}\big(\sum_{j=1}^n \Vert \nabla l_j(\theta_k)\Vert_2\big)^2, ~~\text{for RR}\label{bound_2},
\end{align}
and moreover,
\begin{align}
    \E_{\mathcal{D}}[\xi_{k,j}^2\one_{H_{k-1}}]&\leq 2\Vert \theta_k-\theta_*\Vert_2^2(\sup_j \Vert \nabla l_j(\theta_k)\Vert_2^2+\Vert \nabla L(\theta_k)\Vert_2^2),\quad \text{for SGD}, \\
    \E_{\mathcal{D}^{(w)}}[\xi_{k,j}^2\one_{H_{k-1}}]&\leq 2\Vert \theta_k-\theta_*\Vert_2^2(\frac{1}{n^2}\big(\sum_{j=1}^n \Vert \nabla l_j(\theta_k)\Vert_2\big)^2+\Vert \nabla L(\theta_k)\Vert_2^2),\quad\text{for RR}.
\end{align}
Putting together with other terms, we have (D.12) modified to be
\begin{equation}\label{iteration_1}
\begin{aligned}
    \E_{\mathcal{D}}[R_k\one_{H_{k-1}}]&\leq \E_{\mathcal{D}}[R_{k-1}\one_{H_{k-2}}]-\eps \P(H_{k-2}\setminus H_{k-1})\\
    &+\big((8\Vert \theta_k-\theta_*\Vert_2^2+1)\sup_j \Vert \nabla l_j(\theta_k)\Vert_2^2+8\Vert \theta_k-\theta_*\Vert_2^2\Vert \nabla L(\theta_k)\Vert_2^2 \big)\eta_k^2
    \end{aligned}
\end{equation}
for SGD, and
\begin{equation}\label{iteration_2}
\begin{aligned}
    \E_{\mathcal{D}^{(w)}}[R_k\one_{H_{k-1}}]&\leq \E_{\mathcal{D}^{(w)}}[R_{k-1}\one_{H_{k-2}}]-\eps \P(H_{k-2}\setminus H_{k-1})\\
    &+\big((8\Vert \theta_k-\theta_*\Vert_2^2+1)\frac{1}{n^2}\big(\sum_{j=1}^n \Vert \nabla l_j(\theta_k)\Vert_2\big)^2+8\Vert \theta_k-\theta_*\Vert_2^2\Vert \nabla L(\theta_k)\Vert_2^2 \big)\eta_k^2
    \end{aligned}
\end{equation}
for the RR scheme.

One should control the probability of escaping the neighborhood of $\theta_*$. Denote $r_{\mathcal{U}} =\sup_{\theta\in\mathcal{U}}\Vert \theta-\theta_*\Vert_2$, by taking the telescoping sums of (\ref{iteration_1}) and (\ref{iteration_2}), we get that
\begin{align}
   \E_{\mathcal{D}}[R_k\one_{H_{k-1}}]&\leq  \big((8r_{\mathcal{U}}^2+1)\sup_j \sigma_j^2+8r_{\mathcal{U}}^2 \sum_{i=1}^n \sigma_i^2\big)\sum_{l=1}^k \eta_l^2-\eps \sum_{l=1}^k\P(H_{l-2}\setminus H_{l-1})
\end{align}
for SGD, and
\begin{align}
   \E_{\mathcal{D}^{(w)}}[R_k\one_{H_{k-1}}]&\leq  \big((8r_{\mathcal{U}}^2+1)\frac{1}{n}\sum_{i=1}^n \sigma_i^2 +8r_{\mathcal{U}}^2 \sum_{i=1}^n \sigma_i^2\big)\sum_{l=1}^k \eta_l^2-\eps \sum_{l=1}^k\P(H_{l-2}\setminus H_{l-1})
\end{align}
for the RR scheme. Since the left hand sides are non-negative, we get that
\begin{align}
    \sum_{l=1}^k\P(H_{l-1}\setminus H_{l})\leq \frac{R_*}{\eps}\sum_{l=1}^{\infty} \eta_l^2= \frac{R_* \gamma^2}{\eps}\sum_{l=1}^{\infty} \frac{1}{(l+m)^{2p}},
\end{align}
with $R_*=(8r_{\mathcal{U}}^2+1)\sup_j \sigma_j^2+8r_{\mathcal{U}}^2 \sum_{i=1}^n \sigma_i^2$ for SGD and $R_*=(8r_{\mathcal{U}}^2+1)\frac{1}{n}\sum_{i=1}^n \sigma_i^2 +8r_{\mathcal{U}}^2 \sum_{i=1}^n \sigma_i^2$ for RR scheme, respectively. Because $p>1/2$, the sum on the right hand side is finite. For any $\delta\in(0,1)$, we can choose $m$ sufficiently large, so that $\frac{R_* \gamma^2}{\eps}\sum_{l=1}^{\infty} \frac{1}{(l+m)^{2p}}<\delta$. With that, for any $k$, we get
\begin{align}
    \P(H_k) = 1-\sum_{l=1}^k\P(H_{l-1}\setminus H_{l})\geq 1-\delta.
\end{align}
As a consequence, since $E_{\infty} = \bigcap_{k=1}^{\infty}E_k$, we will obtain
\begin{align}
    \P(E_{\infty})= \inf_k \P(E_k)\geq \inf_k \P(H_{k-1})\geq 1-\delta.
\end{align}
For the convergence rate, from (\ref{aug261}) we have
\begin{equation}
    \begin{aligned}
        \E[D_{k+1}\one_{E_{k+1}}]\leq \E[D_{k+1}\one_{E_k}]\leq (1-2\alpha \eta_k )\E[D_k\one_{E_k}]  + \eta_k^2 \E[\Vert V_{k}\Vert_2^2\one_{E_k}].
    \end{aligned}
\end{equation}
The estimate of the second term above is similar to (\ref{bound_1}) and (\ref{bound_2}). Now we insert the expression of $\eta_k$ and only consider $p=1$ for shortness, since the result can be derived for $p<1$ verbatim, we get that
\begin{align}
    \E_{\mathcal{D}}[D_{k+1}\one_{E_{k+1}}]&\leq \big(1-\frac{2\alpha \gamma}{k+m} \big)\E_{\mathcal{D}}[D_k\one_{E_k}]  + \frac{\gamma^2}{(k+m)^2} \sup_j \sigma_j^2, \quad \text{ for SGD,}\\
    \E_{\mathcal{D}^{(w)}}[D_{k+1}\one_{E_{k+1}}]&\leq \big(1-\frac{2\alpha \gamma}{k+m} \big)\E_{\mathcal{D}^{(w)}}[D_k\one_{E_k}]  + \frac{\gamma^2}{n(k+m)^2} \sum_{i=1}^n \sigma_i^2, \quad \text{ for RR}.
\end{align}
Therefore, we eventually get that, with $\gamma$ large so that $2\alpha\gamma>1$,
\begin{align}\label{result_1}
   \E_{\mathcal{D}}[\Vert\theta_k-\theta_*\Vert_2^2|E_{\infty}]\leq \frac{\E_{\mathcal{D}}[D_k\one_{E_{\infty}}] }{\P(E_{\infty})}\leq \frac{\gamma^2}{(1-\delta)(2\alpha \gamma-1)k}\sup_j \sigma_j^2+o(1/k)
\end{align}
for SGD, and 
\begin{align}\label{result_2}
   \E_{\mathcal{D}^{(w)}}[\Vert\theta_k-\theta_*\Vert_2^2|E_{\infty}]\leq \frac{\gamma^2}{(1-\delta)(2\alpha \gamma-1)k}\left(\frac{1}{n}\sum_{i=1}^n \sigma_i^2\right)+o(1/k)
\end{align}
for the RR scheme.
\end{proof}
\begin{remark}
The convergence rate comparison result is similar to \cite{needell2014stochastic}, in the spirit that for loss functions with drastically different subfunction slopes, the RR scheme performs as averaging to speed up the convergence rate. We should point it out that the analysis in \cite{needell2014stochastic} only considers strong convex objectives, and the weighted SGD they investigate has weights proportional to Lipschitz constants of $\nabla l_i$ rather than $\Vert \nabla l_i\Vert_2$.
\end{remark}

\subsection{Comments from the asymptotic viewpoint}
With the derivations of (\ref{sme1}) and (\ref{sme2}), one would hope to leverage stochastic calculus tools to give a short and illustrative picture for stochastic algorithm behavior comparisons. We make following remarks which fully use the trace information of covariance matrices, which may indicate the faster local convergence and less oscillation behaviors of the RR scheme.

\begin{remark}[Local convergence rate in the continuous-time limit]\label{rmk:conv}
Suppose $\theta_*$ is a local minimum, and there exists $r>0$ such that the stochastic trajectories
$\Theta_t$ with the starting point $\Theta_0=\theta_0\in B(\theta_*,r)$ stay inside the ball
$B(\theta_*,r)$ for $t\geq 0$. Also the local strong convexity holds: there exists $c_0>0$ such that
$(x-y)^{\top}(\nabla L(x)-\nabla L(y))\geq c_0\Vert x-y \Vert_2^2$ for $x,y \in B(\theta_*,r)$, then
the trajectory driven by (\ref{sme2}) converges faster to $\theta_*$ compared to the
trajectory driven by (\ref{sme1}) in the $L^2(0,t;L^2(\R^d))$-sense.
\end{remark}

Let us elaborate on the Remark \ref{rmk:conv}. We denote the solutions to \eqref{sme1} and \eqref{sme2} as $\Theta^1$ and $\Theta^2$ respectively and compute their convergence rate to the local minimum in $L^2(0,t;L^2(\R^d))$. 
Consider the function $f(x)
=\frac{1}{2} (x-\theta_*)^{\top}(x-\theta_*)$. By Ito's lemma,
\begin{align*}
  df(\Theta^i_t ) = \Big(-(\Theta^i_t-\theta_*)^{\top} m(\Theta^i_t ) +\frac{\eta
  }{2}\Tr(\Sigma_i(\Theta^i_t))\Big) dt + (\Theta^i_t-\theta_*)^{\top}\sigma_i(\Theta^i_t) dW_t
\end{align*}
for $i=1,2$. Therefore, for $i=1,2$,
\begin{align*}
  \E[\Vert\Theta^i_t-\theta_*\Vert_2^2] &= \Vert\theta_0-\theta_*\Vert_2^2 +
  \E\Big[\int_0^t\Big( -(\Theta^i_s-\theta_*)^{\top}\nabla L(\Theta^i_s)+\frac{\eta
    }{2}\Tr(\Sigma_i(\Theta^i_s))\Big)ds\Big]\\ &\leq \Vert\theta_0-\theta_*\Vert_2^2 + \frac{\eta
  }{2}\int_0^t \E\left[\Tr(\Sigma_i(\Theta^i_s))\right]ds - c_0
  \int_0^t\E[\Vert\Theta^i_s-\theta_*\Vert_2^2] ds,
\end{align*}
where the last line uses the local strong convexity and $\nabla L(\theta_*)=0$. Then we deduce that
\begin{align*}
  \int_0^t \E[\Vert\Theta^i_s-\theta_*\Vert_2^2]ds\leq \int_0^t
  e^{-c_0(t-s)}\left(\Vert\theta_0-\theta_*\Vert_2^2 + \frac{\eta }{2}\int_0^s
  \E\left[\Tr(\Sigma^i(\Theta^i_{\tau}))\right]d\tau\right)ds.
\end{align*}
Unfortunately, a direct comparison between two independent stochastic processes is unclear and out of our reach, but we speculate that in order to obtain a faster convergence to the global minimum, it should
have a smaller trace of the covariance matrix. Indeed, by Cauchy-Schwarz, the covariance matrix from
(\ref{sme2}) has a smaller trace at every small neighborhood of points coinciding with the
trajectory from (\ref{sme1}) since
\begin{equation}\label{CSineq}
\begin{aligned}
  \Tr(\Sigma^2(\Theta_t)) = \frac{Z(\Theta_t)}{n^2} &\sum_{i=1}^n \Vert\nabla l_i(\Theta_t)\Vert_2 - \Vert m(\Theta_t)\Vert_2^2\\
    &\leq \frac{1}{n}\sum_{i=1}^n \Vert\nabla l_i(\Theta_t)\Vert_2^2 -\Vert m(\Theta_t)\Vert_2^2 =  \Tr(\Sigma^1(\Theta_t)).
    \end{aligned}
\end{equation}
Back to the discrete algorithms, the above computations also indicate that the RR scheme (\ref{IpSamp}) improves the convergence rate to a local minimum.

\begin{remark}[Deviation from the deterministic gradient descent]
Suppose that $\nabla L$ is uniformly Lipschitz continuous, i.e., there exists $B>0$ such that $\Vert\nabla L(x)-\nabla L(y)\Vert_2\leq B\Vert x-y\Vert_2$ for all $x,y\in\R^d$. Let $\phi_t$ denote the solution to
\begin{align*}
    \dot \phi_t = -\nabla L(\phi_t)
\end{align*}
with $\phi_0=\theta$, and $\Theta_t$ be the solution to
\begin{align*}
   d \Theta_t = -\nabla L(\Theta_t)+\sqrt{\eta}\sigma(\Theta_t)dW_t
\end{align*}
with $\Theta_0=\theta$, then we have an upper bound on the probability of deviation
\end{remark}
\begin{lemma}\label{lem:dev}
For any $\delta>0$ and $0<T<\infty$, we have the inequality
\begin{align}\label{devbound}
    \P_{\theta}\left(\sup_{t\in[0,T]}\Vert \Theta_t-\phi_t\Vert_2>\delta\right)\leq \eta c' \E_{\theta}\left[\int_0^T \Tr(\sigma(\Theta_s)\sigma(\Theta_s)^{\top})ds\right],
\end{align}
where $c'$ only depends on $\delta, T$ and $B$. 
\end{lemma}
The proof of inequality (\ref{devbound}) can be found in the Appendix \ref{sec:appendixB}. Now if we consider $\Theta_t^1$ in (\ref{sme1}) for SGD along with $\Theta_t^2$ in (\ref{sme2}) for the RR scheme, the trace bound in Lemma \ref{lem:dev} implies that (\ref{sme2}) is closer to $\phi_t$ with a higher probability, simply by taking (\ref{CSineq}). In terms of discrete stochastic algorithms, it also indicates that the RR scheme is more deterministic compared to SGD.

\section{Experiments}\label{sec:experiments}
The analysis above suggests that, for non-convex optimization problems, the RR scheme is more
likely to find the global minimum compared to the vanilla SGD, especially when the global minimum
lies in the sharp valley. We empirically verify this in several examples from both data sciences and
physical sciences. In particular, we study (1) robust statistics problems under data feature
disparities, and (2) geometric optimization problems in computational chemistry.

\subsection{Robust classification/regression}
In this example, we consider to use the Welsch loss (see \cite{barron2019general}) from robust
statistics, where the corresponding loss for each data sample $i$ is
\begin{align}
  l_i(\theta) = 1-\exp(-(y_i-\theta^{\top}x_i)^2/2).
\end{align}
This loss function has found many applications in regression problems dealing with
outliers. Suppose that the goal is to find a global optimizer for a mixed population of multiple
subgroups: part of them are quite sensitive in a certain trait, while the rest are much less
sensitive in the same trait. In the following two examples, the RR scheme randomly selects the 
sub-population $j$ with the probability proportional to $a_j \Vert \nabla
l_j(\theta_k)\Vert_2$ with replacement in each iteration, where $a_j$ denotes the sub-population
proportion.

\begin{figure}[h!]
  \centering
  \begin{tabular}{ccc}
    \includegraphics[scale=0.21]{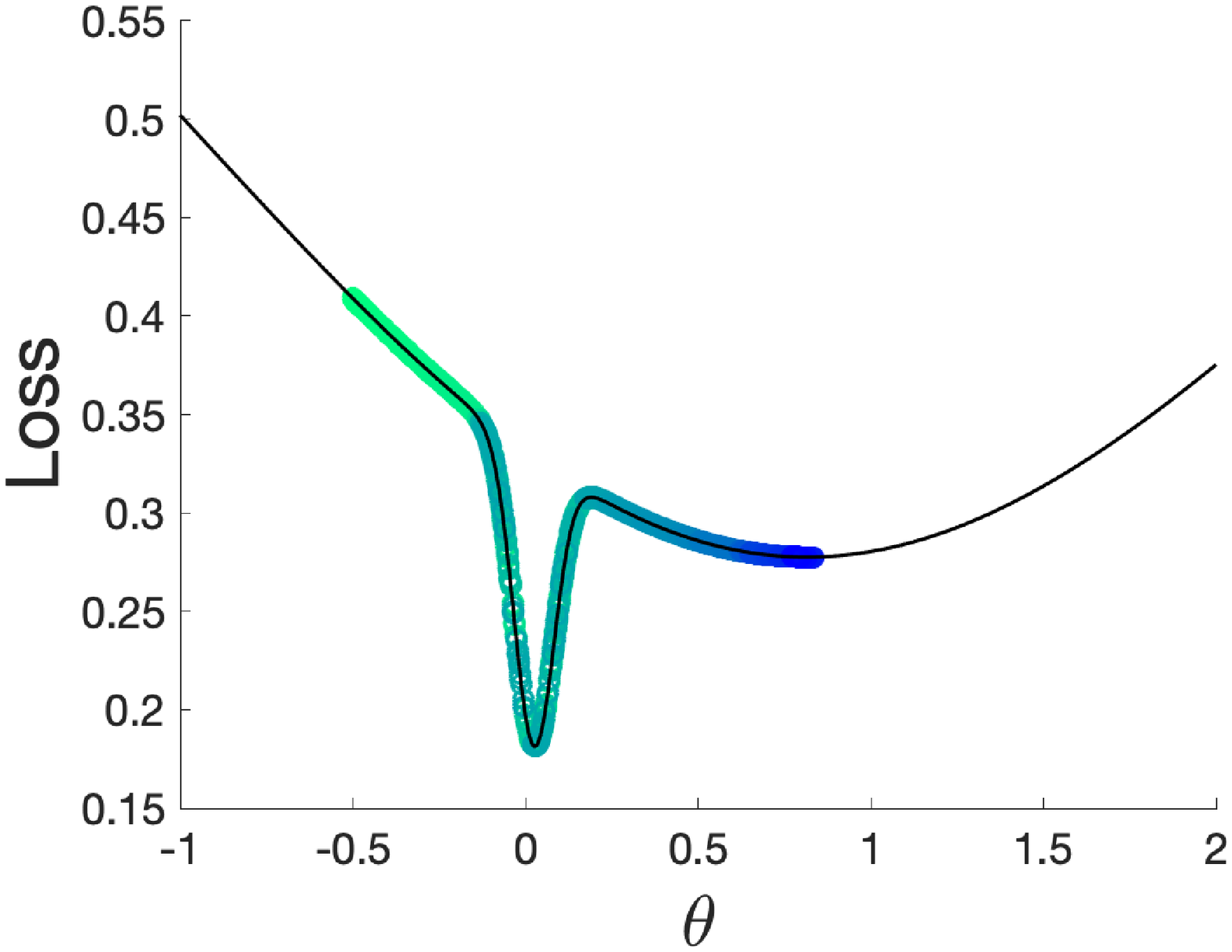}&
    \includegraphics[scale=0.21]{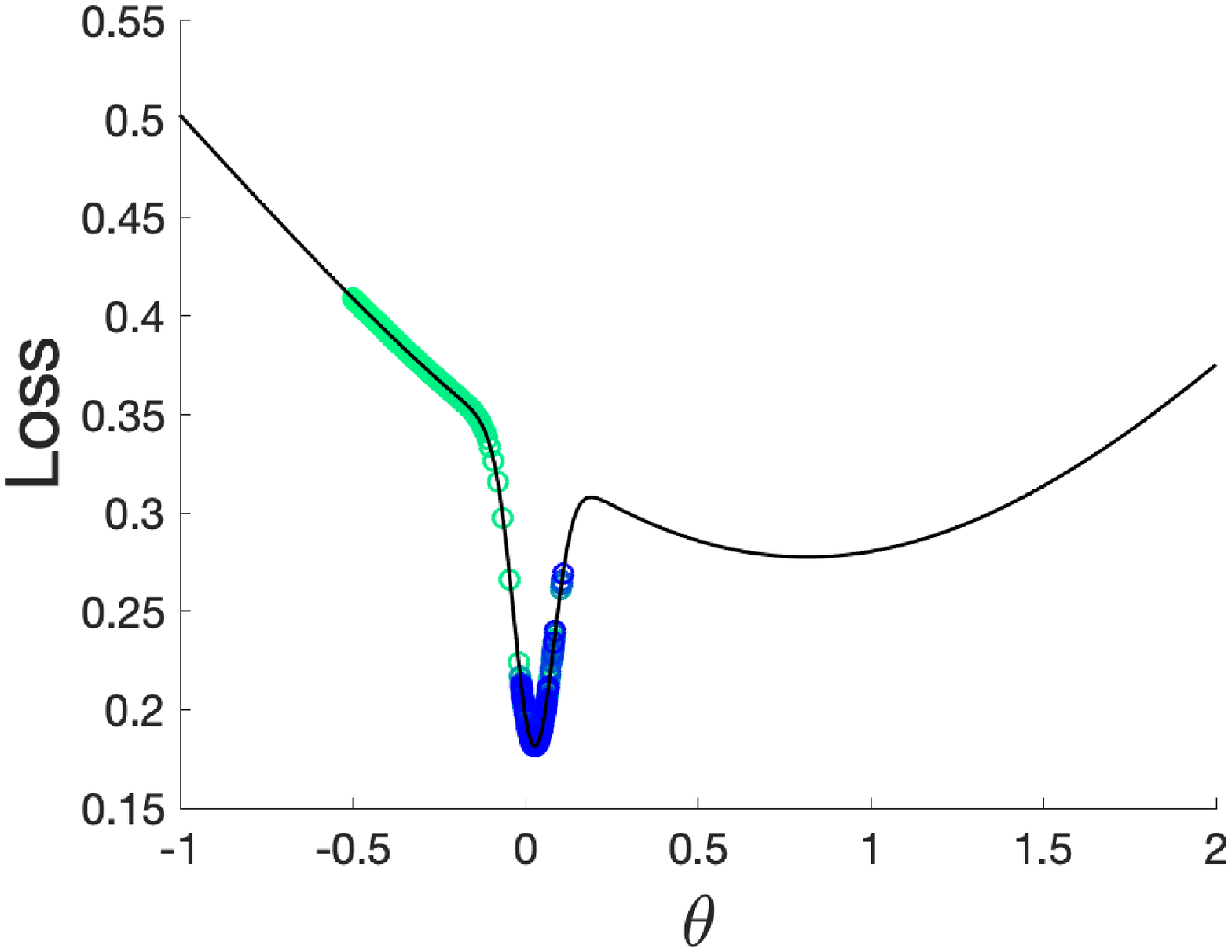} &
    \includegraphics[scale=0.21]{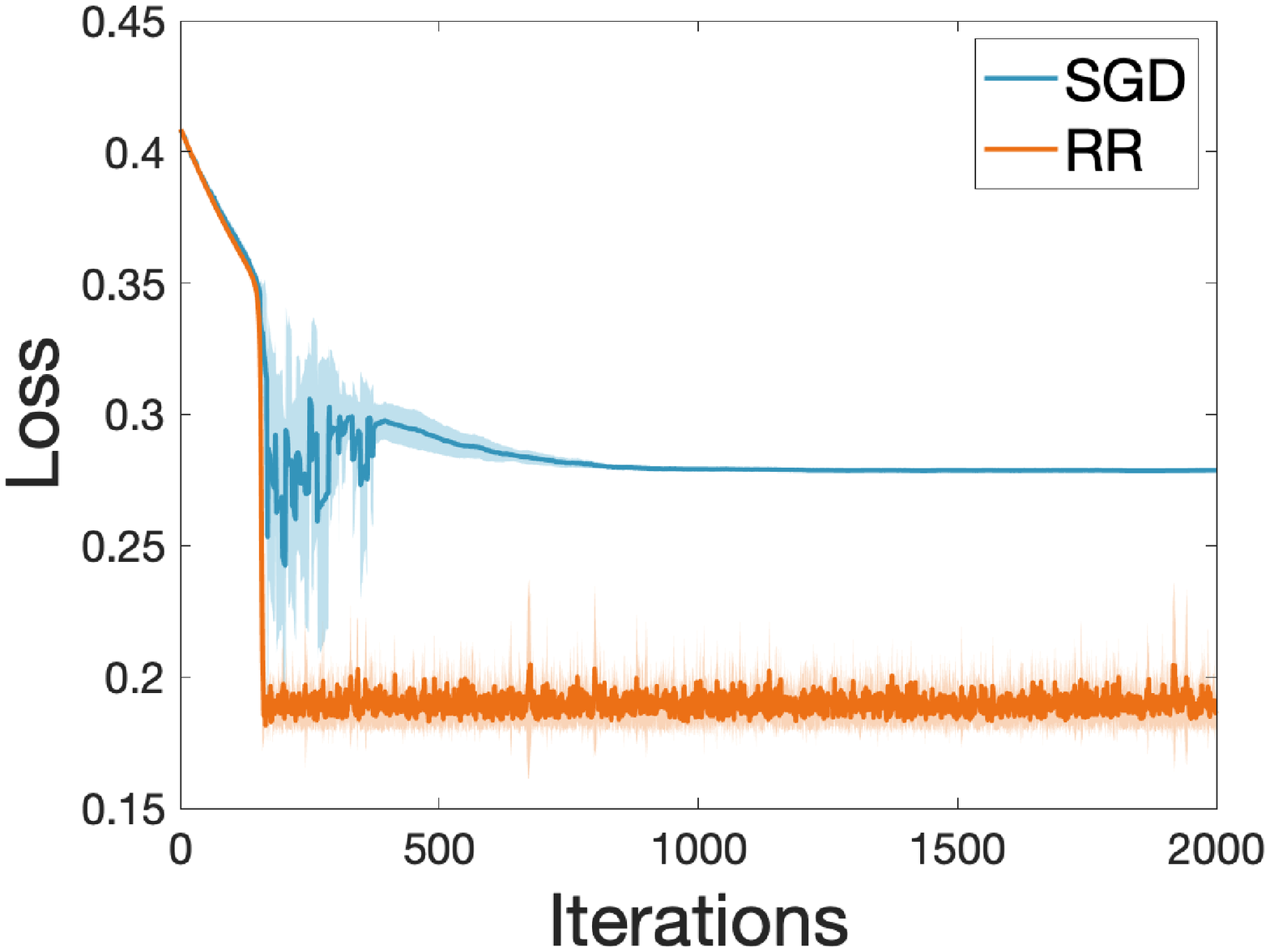}\\
    (1) Vanilla SGD &(2) RR & (3) Loss
  \end{tabular}
  \caption{We set the starting point $\theta_0=-0.5$ with a constant learning rate $\eta = 0.015$ in
    the robust classification problem. The color gradient of circles from green to blue shows how
    iterations proceed. Plots (1) and (2) show the trajectories for one trial under vanilla SGD and
    RR, respectively. Plot (3) show the loss comparisons over $10$ trials.}
  \label{fig:synthetic1}
\end{figure}
\subsubsection{Classification}
The dataset $\{(x_i,y_i)\}_{1\leq i \leq N}, x_i\in\R,y_i\in\{0,1\}$ consists of $x$ from two
subgroups 
\begin{itemize}
\item Subgroup $1$: $x_i = 20 +\mathcal{N}(0,1)$; total number $N_1=800$.
\item Subgroup $2$: $x_i = 0.5 +\mathcal{N}(0,1/4)$; total number $N_2=4000$.
\end{itemize}
Here $a_1/a_2=1/5$. The class $y_i\sim \text{Ber}(1/2)$ is preassigned for each data point, as we assume that each subgroup contains individuals belonging to different classes. The goal is to find the global minimum rather than a local minimum, even though it is flat. From Fig \ref{fig:synthetic1} we can see that in contrast to the vanilla SGD trajectory escapes to the nearby flat local minimum, the RR scheme trajectory stays inside the sharp valley to reach the global minimum.

\subsubsection{Regression} 
The dataset $\{(\textbf{x}_i,y_i)\}_{1\leq i \leq N}, \textbf{x}_i\in\R^d, d=10$ is composed of
samples from two subgroups

\begin{itemize}
\item Subgroup $1$: $\textbf{x}_i = 20\textbf{e} +\mathcal{N}(0,I_d)$; total number $N_1=2000$.
\item Subgroup $2$: $\textbf{x}_i = \frac{1}{4}\textbf{e} +\frac{1}{2}\mathcal{N}(0,I_d)$; total number $N_2=800$.
\end{itemize}
Here $a_1/a_2=5/2$. The exact regression coefficient $\beta^*\in\R^d$ is picked by $\beta^*\sim \mathcal{N}(0,I_d)$ and
the uncorrupted response variables are $y_i^* = \textbf{x}_i^{\top} \beta^*$. The corrupted response
variables are generated by $y_i = y_i^*+u_i+\eps_i$, where $u_i\sim \text{Unif}([-3\lVert
  y^*\rVert_{\infty}, 3\lVert y^*\rVert_{\infty}]), \eps_i\sim\frac{1}{10}\mathcal{N}(0,1)$. Fig \ref{fig:synthetic2} shows that for a wide range of learning rate choices, the RR scheme selects a better minimum in a faster speed compared to SGD.
\begin{figure}[h!]
  \centering
  \begin{tabular}{ccc}
    \includegraphics[scale=0.22]{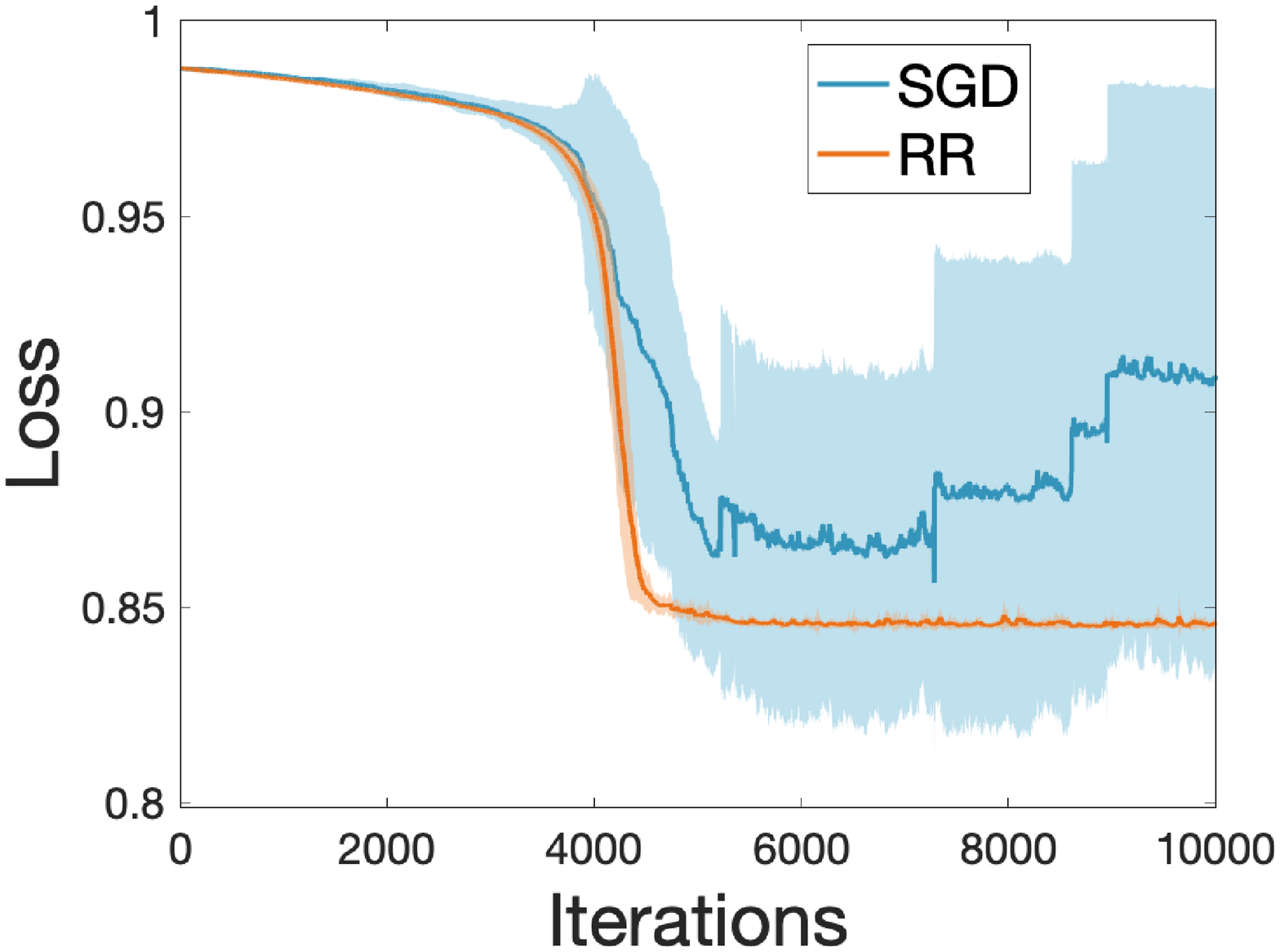}&
    \includegraphics[scale=0.22]{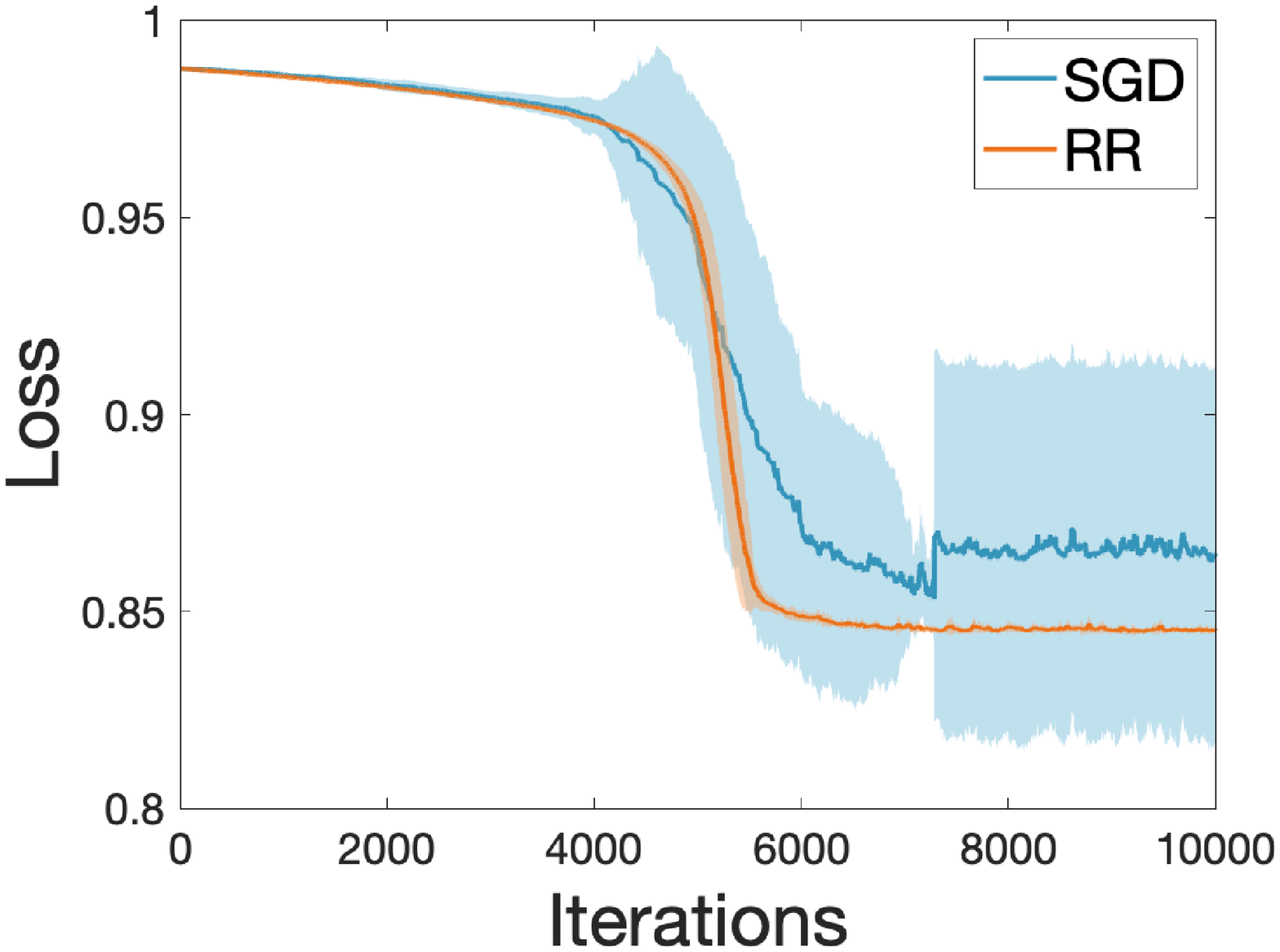}&
    \includegraphics[scale=0.22]{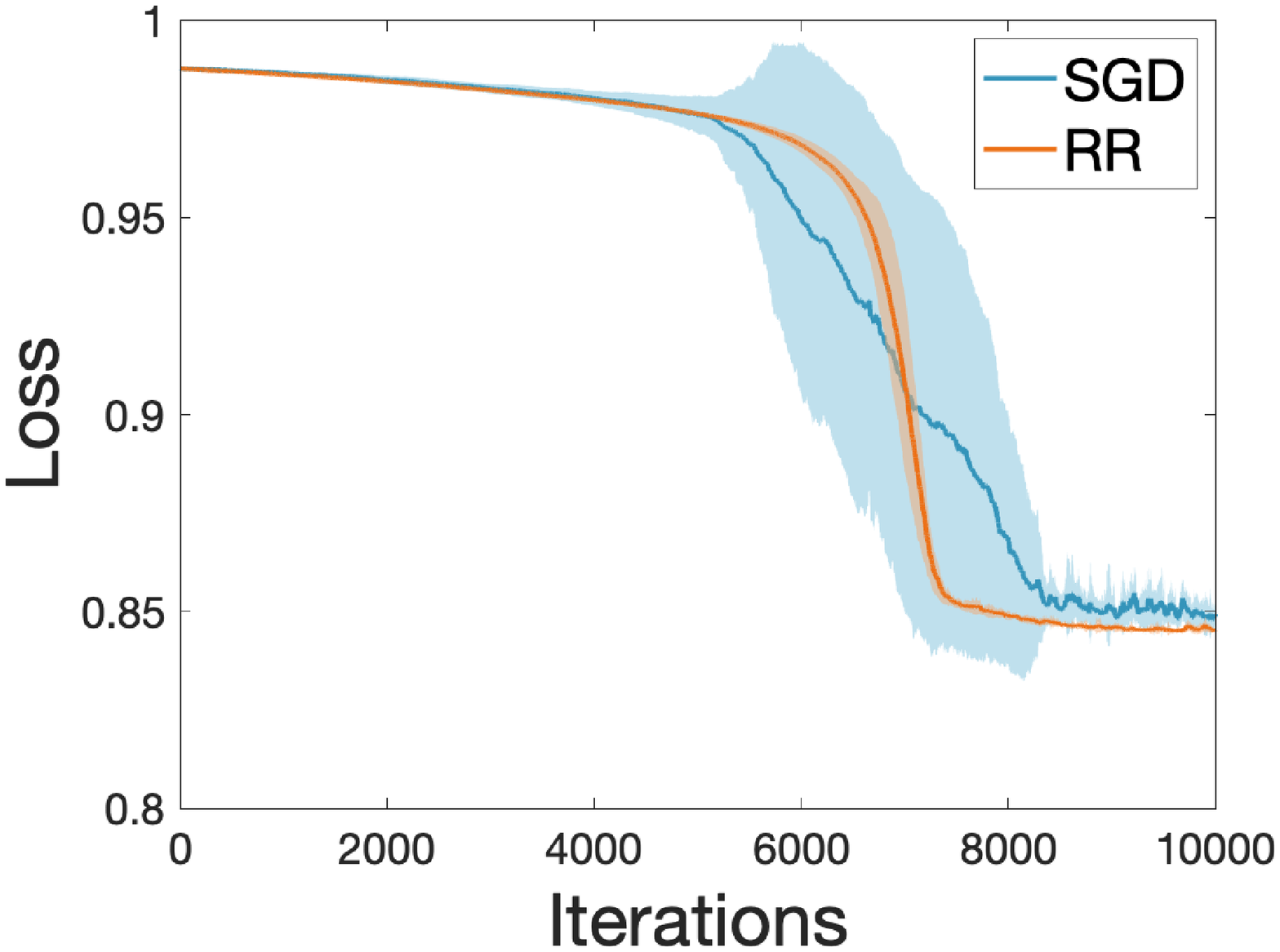} \\
     (1) $\eta=0.5$  &(2) $\eta=0.4$ &(3) $\eta=0.3$ 
  \end{tabular}
  \caption{Comparisons of the vanilla SGD and the RR scheme for various learning rates over 10 trials
    in robust regression. For each plot, we start from the same $\beta_0 =
    \beta^*+5\textbf{e}+{\eps}, {\eps}\sim \mathcal{N}(0,I_d)$. We observe with a relatively large
    learning rate, SGD can deviate while the RR scheme is stable around the local minimum. Even
    when the learning rate decreases and SGD converges to the same minimum that the RR scheme
    arrives, the convergence speed of SGD is slower.}
  \label{fig:synthetic2}
\end{figure}
\subsection{Computational chemistry}\label{chemistry}
The problem here is to find the global minimum of a potential energy surface (PES) that typically
gives a mathematical description of the molecular structure and its energy. In general, assuming
that there are $n$ atoms to form a molecule, we consider minimizing the particle interacting energy
of the form
\begin{align*}
  E(\textbf{z}_1,\textbf{z}_2,\cdots, \textbf{z}_m) = \sum_{i<j}^m V(\textbf{z}_i,\textbf{z}_j),
\end{align*}

where $V$ is a bi-atom potential function, and $\textbf{z}_k\in \R^d, 1\leq k\leq m, d\geq 1$ denote the atom's position. In particular, the global minimum represents the most
stable conformation with respect to location arrangements of atoms. Though using SGD for a large
system is computationally efficient, trying to find the global minimum with SGD can be difficult,
especially when the global minimum lies in a sharp valley. The RR scheme outperforms SGD in terms of
the likelihood of identifying the global minimum. This is demonstrated in Fig \ref{fig:pes} by
looking at two examples. The first one is the M\"{u}ller-Brown potential \cite{muller1979location}.
The second one is an artificial large system with the interacting function of Gaussian type
\begin{align*}
  V_k(\textbf{z}_i,\textbf{z}_k) = \exp\left(-(\textbf{z}_i-\textbf{z}_k)^{\top}M_k (\textbf{z}_i-\textbf{z}_k)\right),\quad M_k = \begin{bmatrix}
    a_k & b_k/2\\
    b_k/2 & c_k 
  \end{bmatrix},
\end{align*}
with $\textbf{z} = (x,y)\in\R^2$. We assume that except for the atom $i$, the rest of atoms' positions are fixed, then it is to consider minimizing the potential energy $\min_{\textbf{z}_i}
\frac{1}{m-1}\sum_{k\neq i}^{m} V_k(\textbf{z}_i,\textbf{z}_k)$ and finding the optimal $\textbf{z}_i$. Detailed parameters
setups are provided in the Appendix \ref{sec:appendixC}.
\begin{figure}[h!]
  \centering
  \setlength{\tabcolsep}{-4pt}
  \begin{tabular}{cccc}
    \includegraphics[scale=0.19]{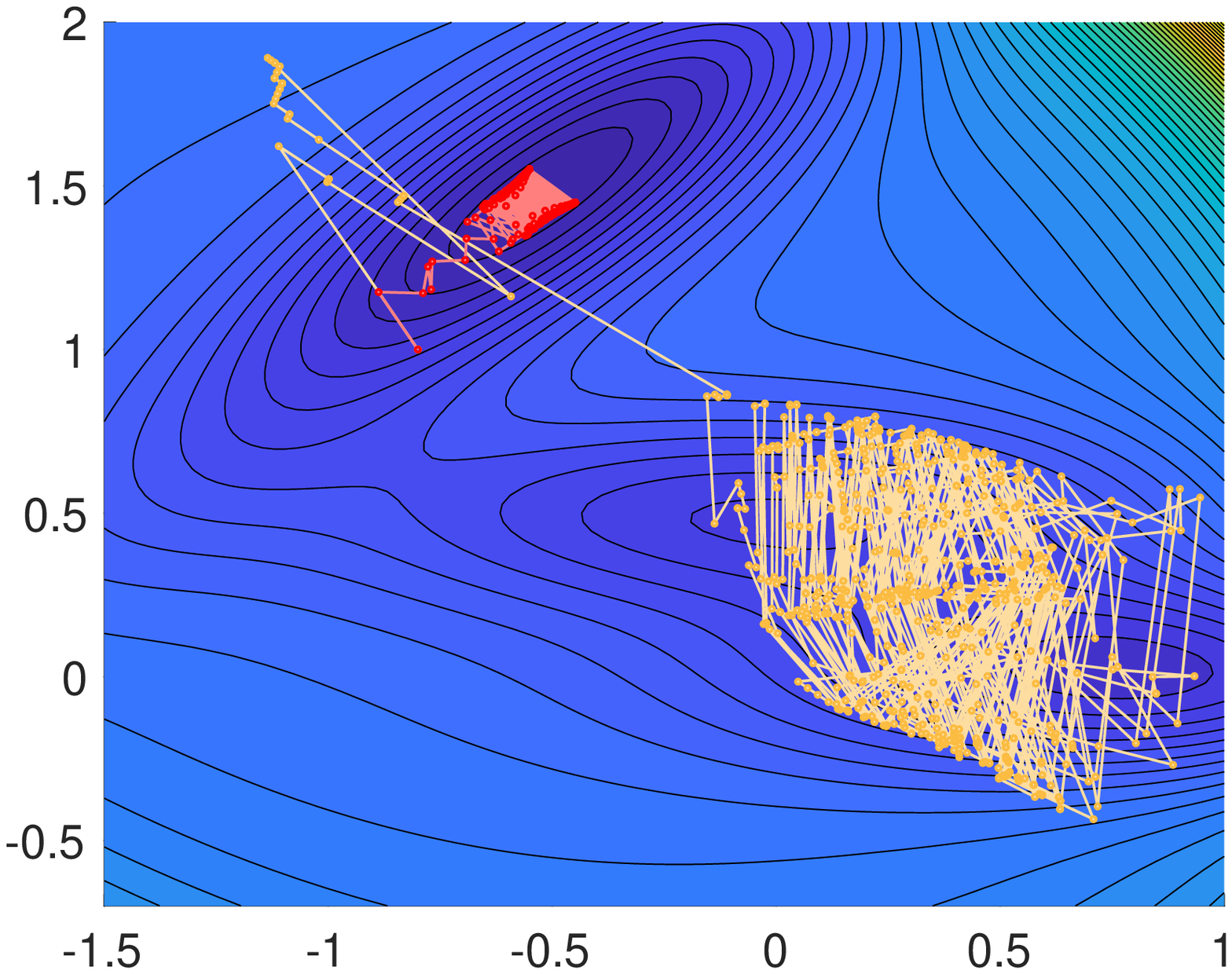}&
    \includegraphics[scale=0.19]{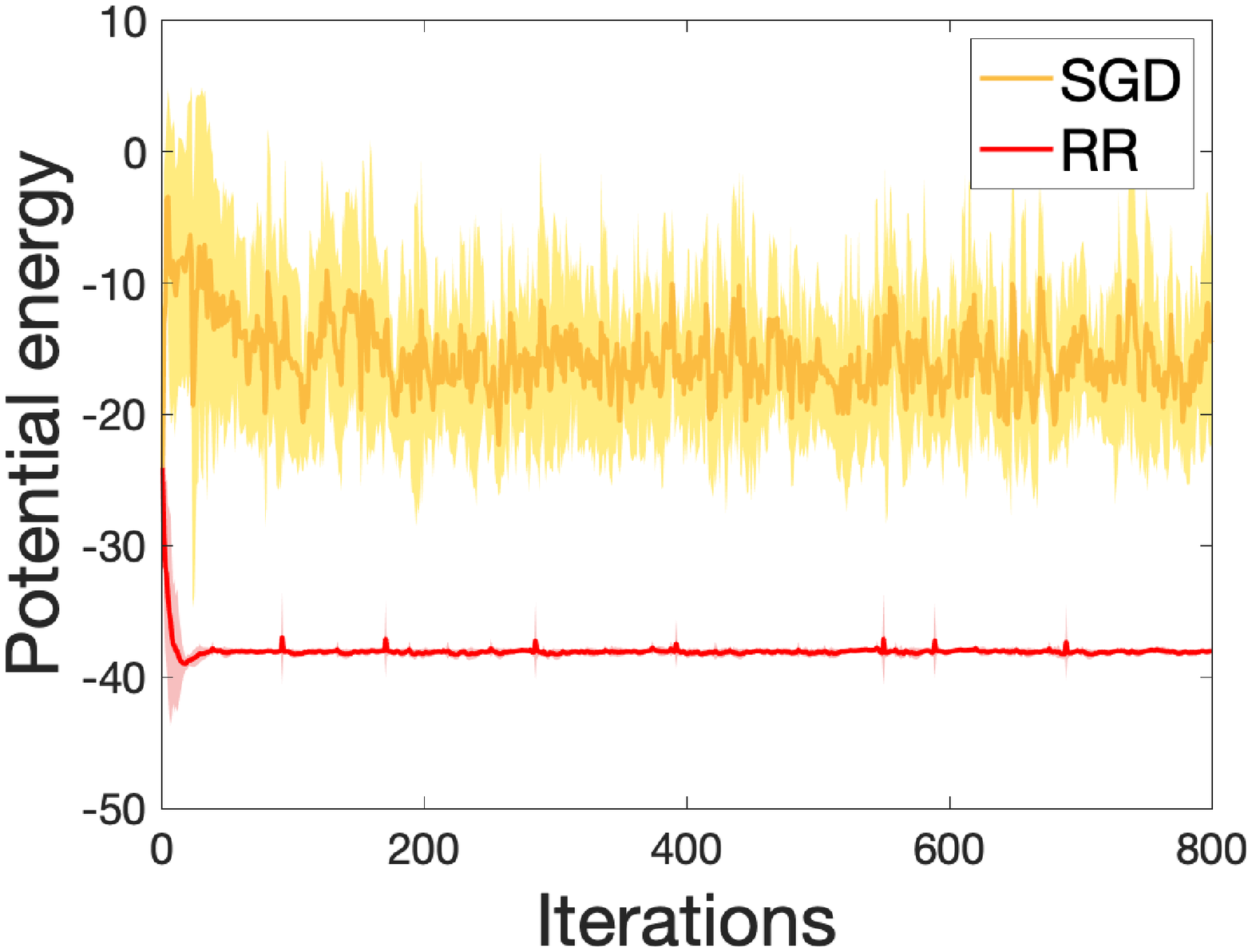}&
    \includegraphics[scale=0.19]{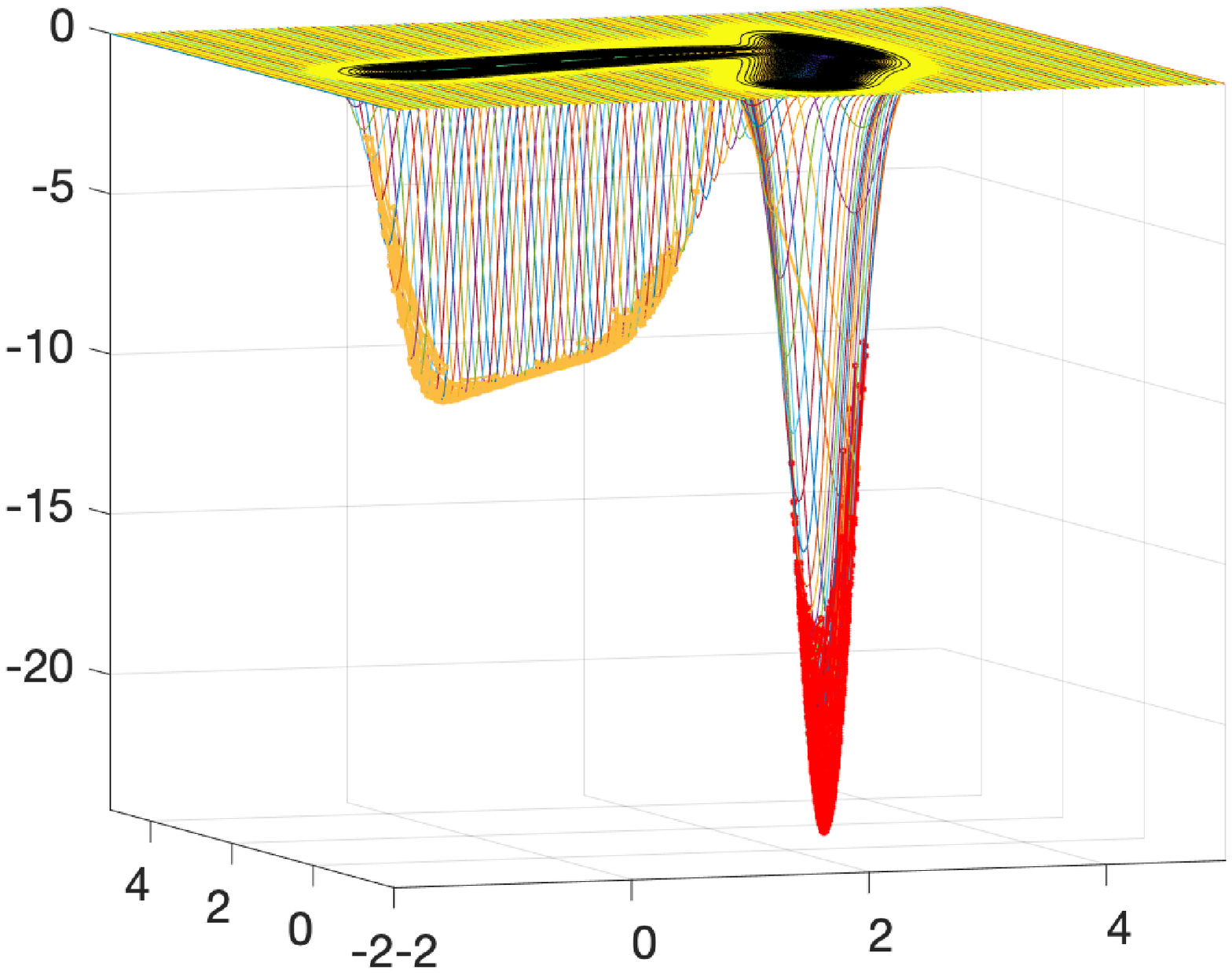}&
    \includegraphics[scale=0.19]{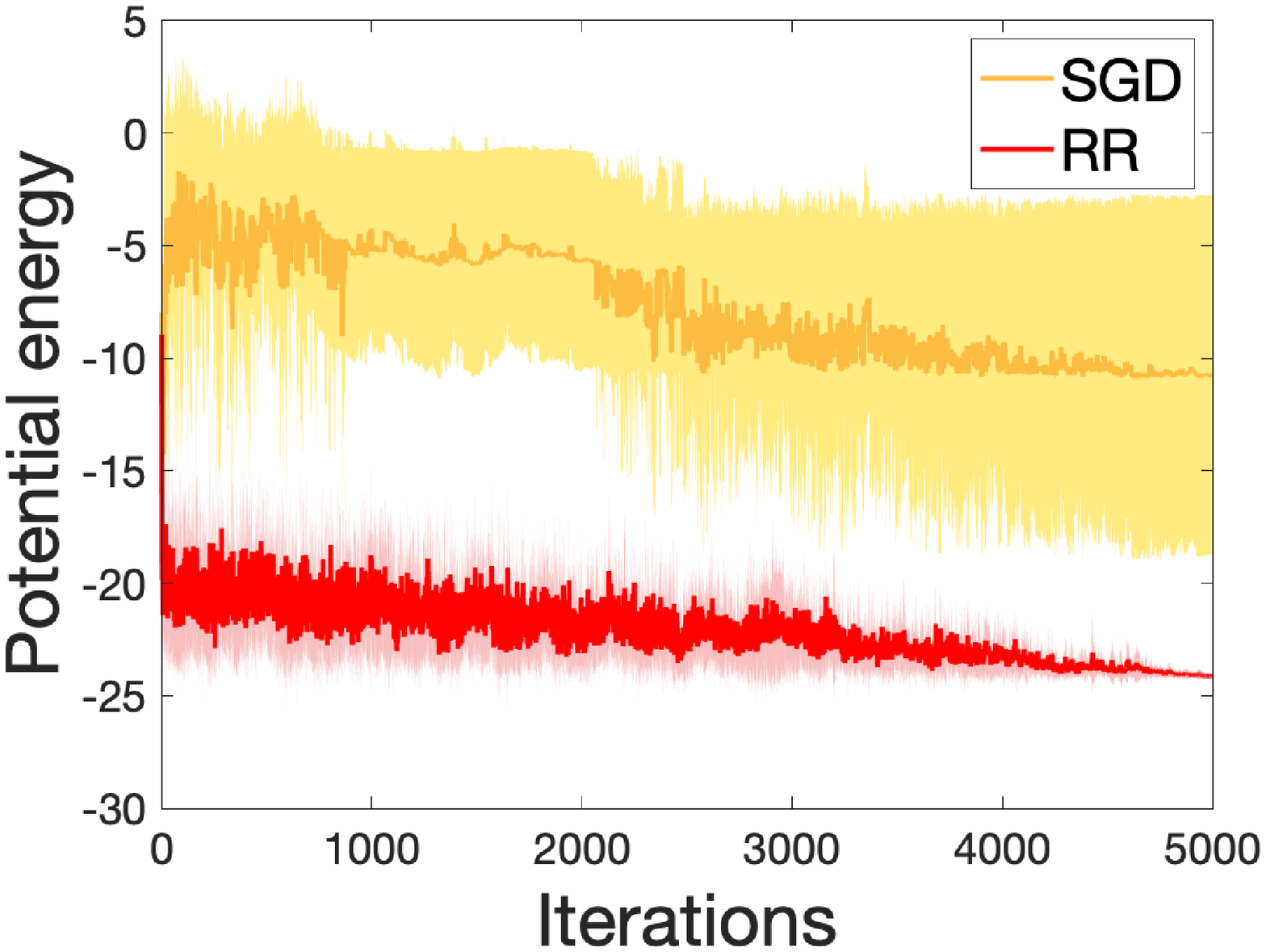}\\
    (1) M\"{u}ller-Brown potential  & & (2) A large system &
  \end{tabular}
  \caption{(1) The global optimization problem for the $1/4$-rescaled M\"{u}ller-Brown potential
    with a learning rate $\eta=0.002$. Both start at $(x_0,y_0) = (-0.8,1.0)$. (2) A large system
    with $m=1000$ atoms in $5000$ iterations. Because this problem is rather stiff, we use a
    monotonically decreasing learning rate starting from $\eta=0.002$ and ends at $\eta=1e-5$ for
    the optimization. Both trajectories start at $(x_0,y_0) = (3.0,1.0)$. For both examples (1) and
    (2), the left plot shows the trajectories for one trial, and the right plot shows the potential
    energy over $10$ trials.}
  \label{fig:pes}
\end{figure}

\section{Conclusion}\label{sec:conclusion}
To deal with data feature disparities in non-convex optimization problems, we propose in this paper a combined resampling-reweighting (RR) scheme to balance variances experienced in different regions. The RR scheme connects with the importance sampling SGD that was previously proposed and analyzed for convex optimization problems. We extend the analysis of the importance sampling SGD to non-convex problems from the viewpoints of stochastic stability and local convergence speed. Numerical experiments verify that the RR scheme outperforms the SGD in capturing the sharp global minimum, making it more reliable and faithful for optimization purposes.

\appendix
\section{Proof of Lemma \ref{lem:1}}\label{sec:appendixA}
\begin{proof}
In different regions, the associated variances are different:
\begin{itemize}
    \item In regions $(-\infty, -1)$, $(-1,0)$,  $(0, 1/K)$, $(1/K, \infty)$.
    \item With probability $a_1$, the gradients are $-1$, $1$, $\eps$, $\eps$ respectively.
    \item With probability $a_2$, the gradients are $-\eps$, $-\eps$, $-K $, $K $ respectively.
    \item Corresponding variances are $a_1a_2(1-\eps)^2$, $a_1a_2(1+\eps)^2$, $a_1a_2(K+\eps)^2$, $a_1a_2(K-\eps)^2$ in each region.
\end{itemize}
Indeed, we take $\eps \to 0$. For $x<0$, the SGD is approximately
\begin{align}
    \theta \leftarrow \theta - V'(\theta)\eta + \sqrt{a_1a_2} \eta,
\end{align}
and the corresponding SDE is
\begin{align}
    d\Theta_t =-V'(\Theta_t) dt + \sqrt{a_1a_2\eta} dW_t.
\end{align}
For $x>0$, the SGD is approximately
\begin{align}
    \theta \leftarrow \theta - V'(\theta)\eta + K\sqrt{a_1a_2} \eta,
\end{align}
and the corresponding SDE is
\begin{align}
    d\Theta_t =-V'(\Theta_t) dt + K\sqrt{a_1a_2\eta} dW_t.
\end{align}
The resulted equilibrium measures on two sides are
\begin{align}
    p(\theta) = \frac{1}{Z_1} \exp\left(-\frac{2}{a_1a_2\eta} V(\theta)\right)\quad\text{for}~\theta<0,\quad p(\theta) = \frac{1}{Z_2} \exp\left(-\frac{2}{K^2a_1a_2\eta} V(\theta)\right)\quad\text{for}~\theta>0.
\end{align}
Consider a SDE with non-smooth diffusion 
\begin{align*}
    dX_t = \mu(X_t,t) dt + \sigma(X_t,t) dB_t,
\end{align*}
the corresponding Kolmogorov forward equation is
\begin{align*}
    \d_s p(x,s) = -\d_x[\mu(x,s)p(x,s)]+ \frac{1}{2}\d^2_x[\sigma^2(x,s)p(x,s)],
\end{align*}
for $s\geq t$. For the equilibrium measure $\d_s p(\theta) = 0$. In order to have $\sigma^2 p$ be continuous at $\theta=0$, as $V(0)=0$, it suggests that
\begin{align*}
    \frac{1}{Z_1} a_1 a_2\eta = \frac{1}{Z_2} K^2 a_1 a_2\eta \quad \Longrightarrow \quad Z_2 = K^2 Z_1.
\end{align*}
\end{proof}
\section{Proof of Lemma \ref{lem:dev}}\label{sec:appendixB}
\begin{proof}
Since there exists $B>0$ such that $\Vert \nabla L(x)-\nabla L(y)\Vert_2\leq B\Vert x-y\Vert_2$ for all $x,y\in\R^d$. Then we have
\begin{align*}
    \Vert\Theta_t - \phi_t\Vert_2\leq B\int_0^T\Vert\Theta_s-\phi_s\Vert_2 ds + \sqrt{\eta} \left\Vert\int_0^T \sigma(\Theta_s)dW_s\right\Vert_2.
\end{align*}
By the Gronwall's inequality, we get
\begin{align*}
    \sup_{t\in[0,T]} \Vert\Theta_t - \phi_t\Vert_@ \leq \sqrt{\eta} e^{BT}  \sup_{t\in[0,T]} \left\Vert\int_0^T \sigma(\Theta_s) dW_s\right\Vert_2.
\end{align*}
Therefore, 
\begin{align*}
    \P_{\theta}\left(\sup_{t\in[0,T]}\Vert\Theta_t-\phi_t\Vert_2>\delta\right)&\leq\P_{\theta}\left(\sup_{t\in[0,T]}\left\Vert\int_0^T \sigma(\Theta_s) dW_s\right\Vert_2>\frac{\delta}{\sqrt{\eta}}e^{-BT}\right)\\
        &\leq \frac{\eta}{\delta^2 e^{-2BT}}\E_{\theta}\left[\left(\sup_{t\in[0,T]}\left\Vert\int_0^T \sigma(\Theta_s) dW_s\right\Vert_2\right)^2\right]\\
   &\leq \eta c' \E_{\theta}\left[\int_0^T \Tr(\sigma(\Theta_s)\sigma(\Theta_s)^{\top})ds\right],
\end{align*}
where we use Chebyshev's inequality for the second last inequality and Burkholder-Davis-Gundy maximal inequality for the last inequality.
\end{proof}

\section{Extended numerical results and parameters setup}\label{sec:appendixC}
\subsection{Numerical comparisons with different learning rates}
Here we show more numerical comparisons between SGD and the RR scheme with various learning rates. See Fig \ref{fig:moreimages1}.
\begin{figure}[htb]
    \centering 
\begin{subfigure}{0.3\textwidth}
  \includegraphics[width=\linewidth]{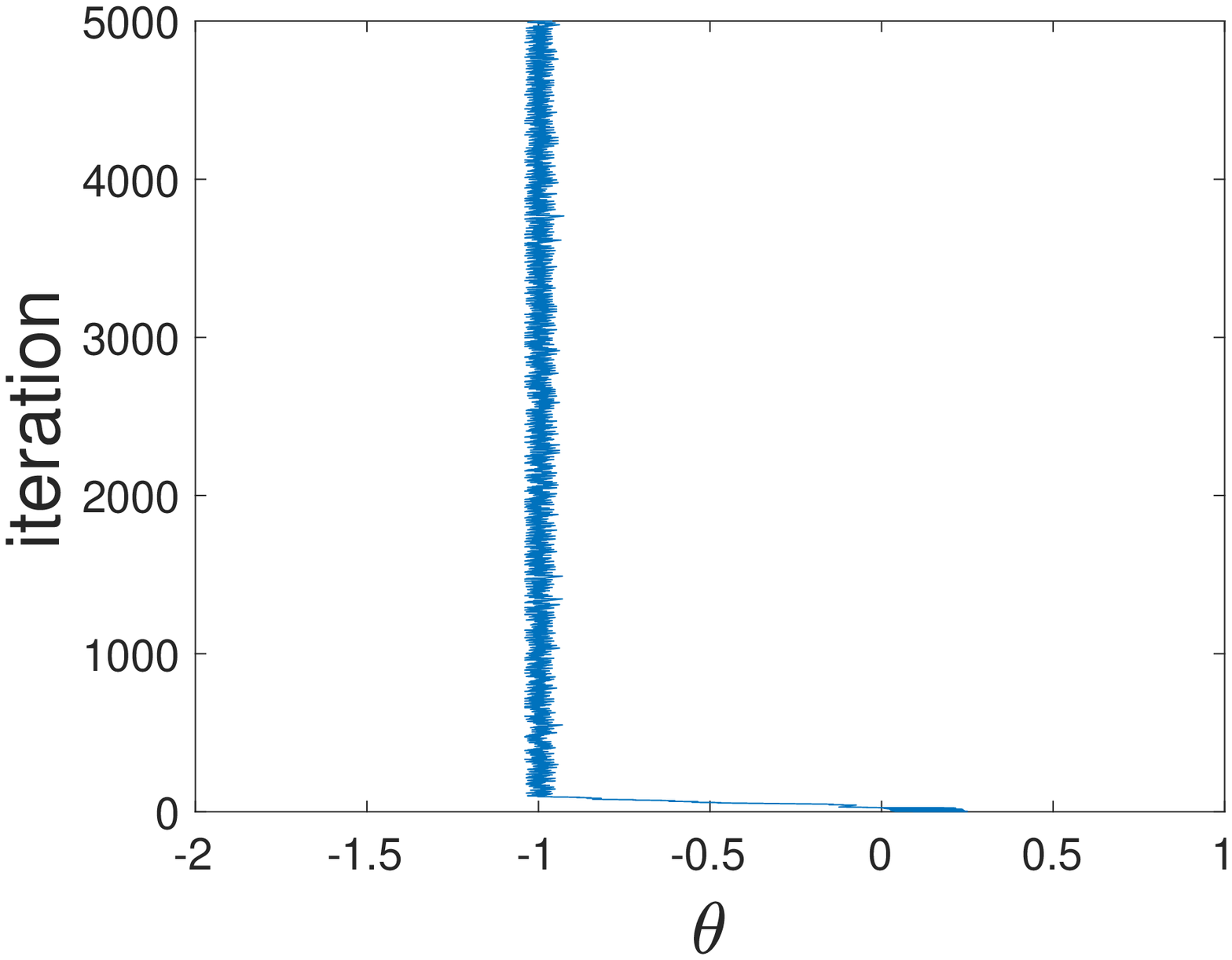}
\end{subfigure}\hfill 
\begin{subfigure}{0.3\textwidth}
  \includegraphics[width=\linewidth]{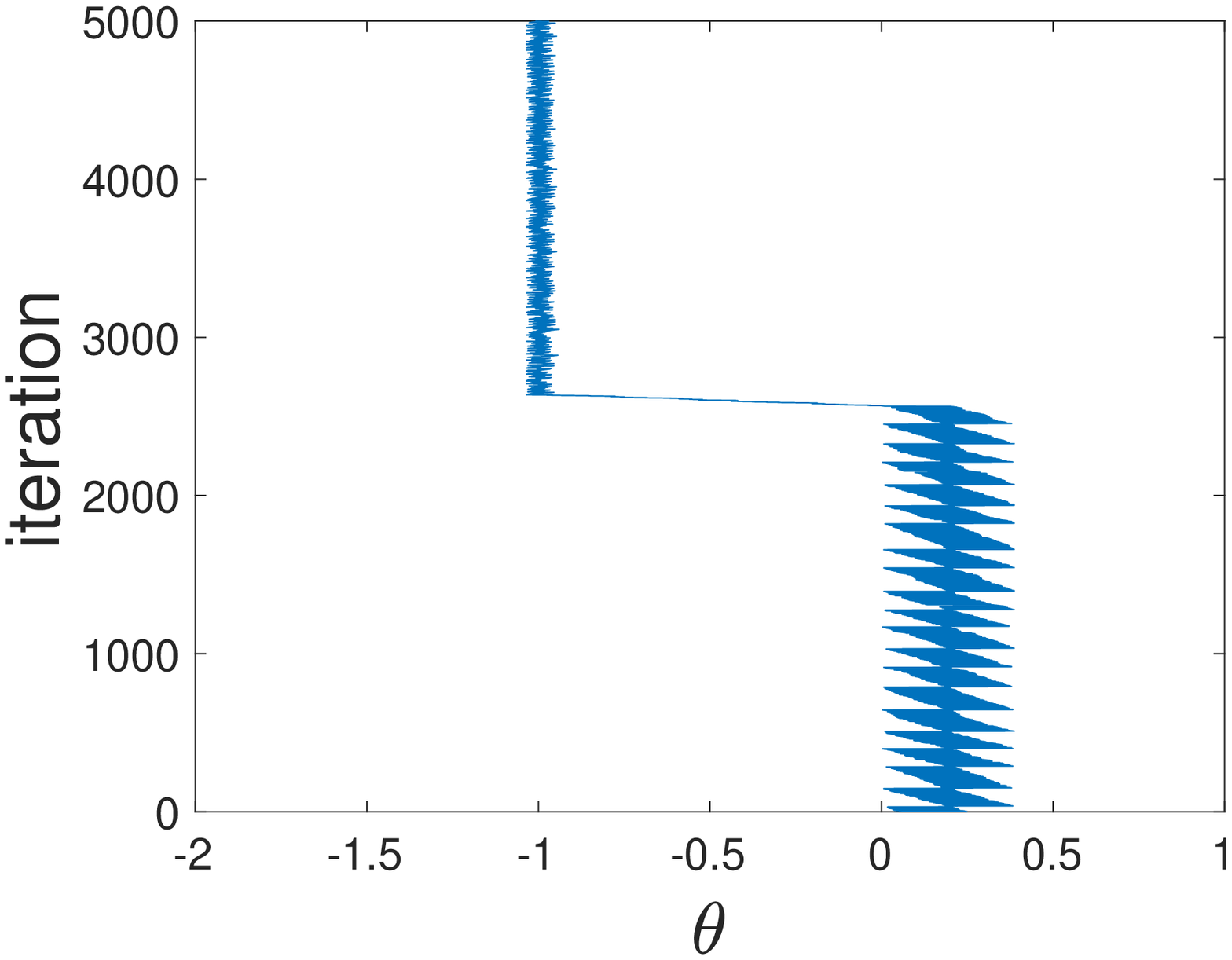}
\end{subfigure}\hfill 
\begin{subfigure}{0.3\textwidth}
  \includegraphics[width=\linewidth]{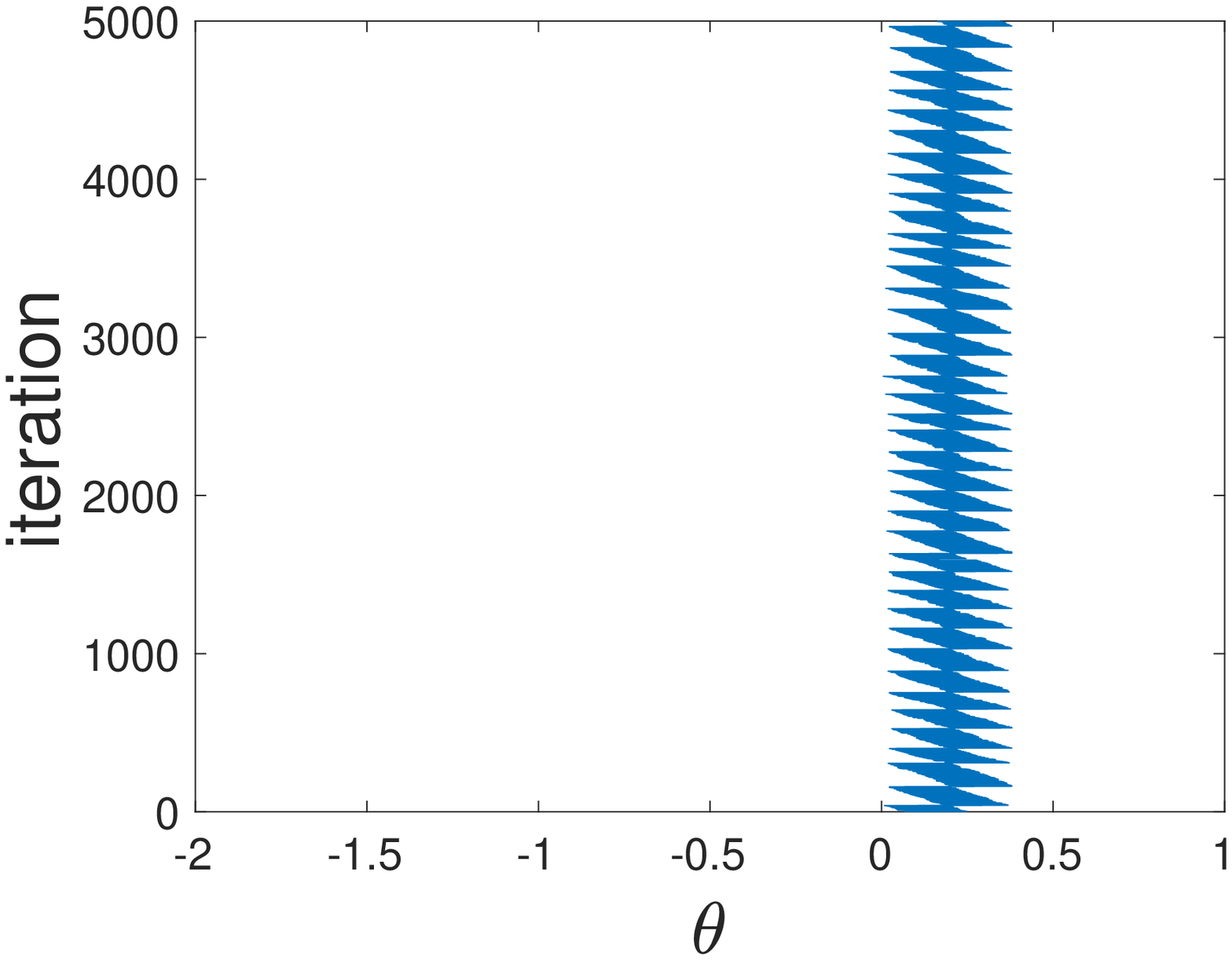}
\end{subfigure}

\medskip
\begin{subfigure}{0.3\textwidth}
  \includegraphics[width=\linewidth]{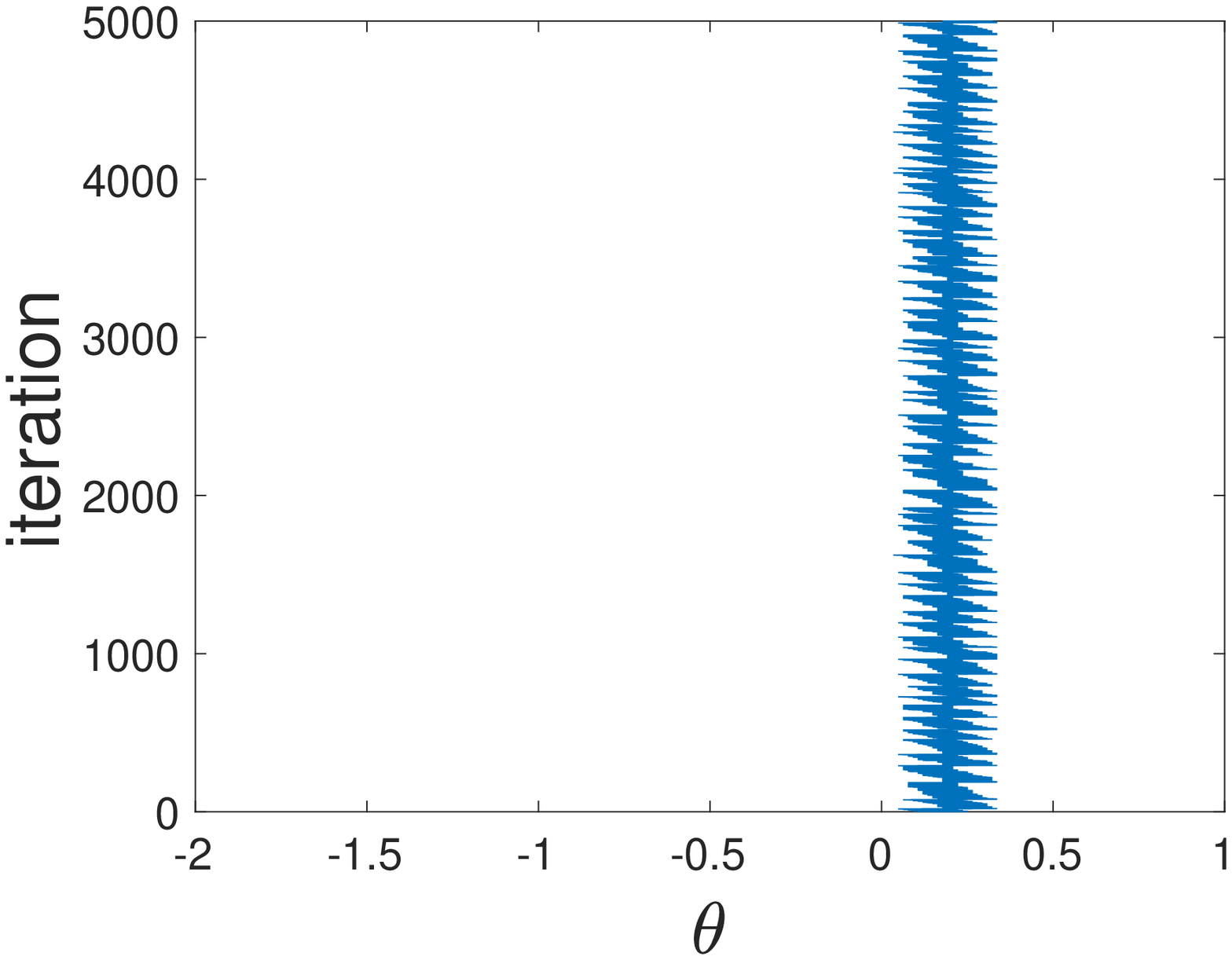}
  \caption{$\eta$ = 0.42}
\end{subfigure}\hfill 
\begin{subfigure}{0.3\textwidth}
  \includegraphics[width=\linewidth]{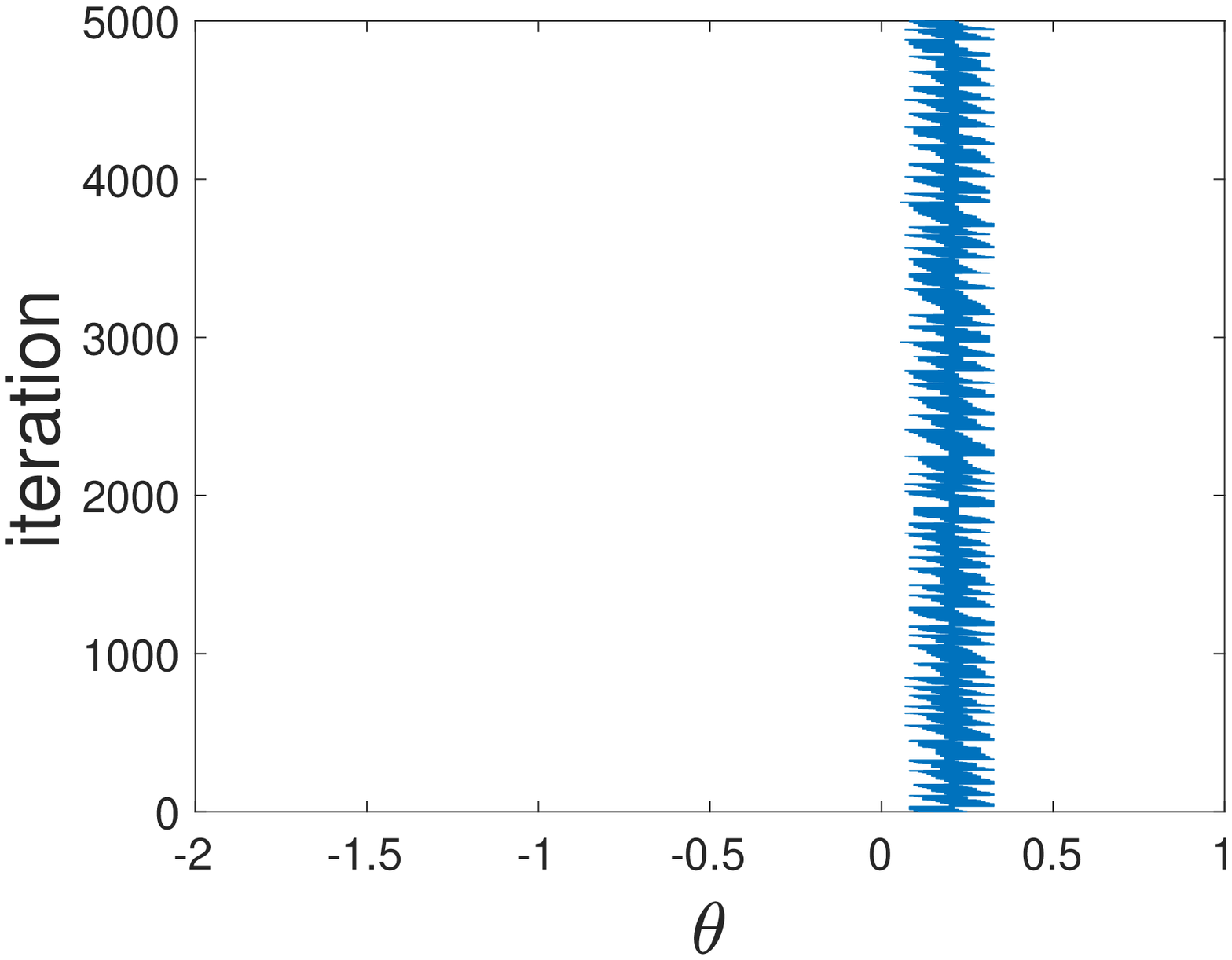}
  \caption{$\eta$ = 0.38}
\end{subfigure}\hfill 
\begin{subfigure}{0.3\textwidth}
  \includegraphics[width=\linewidth]{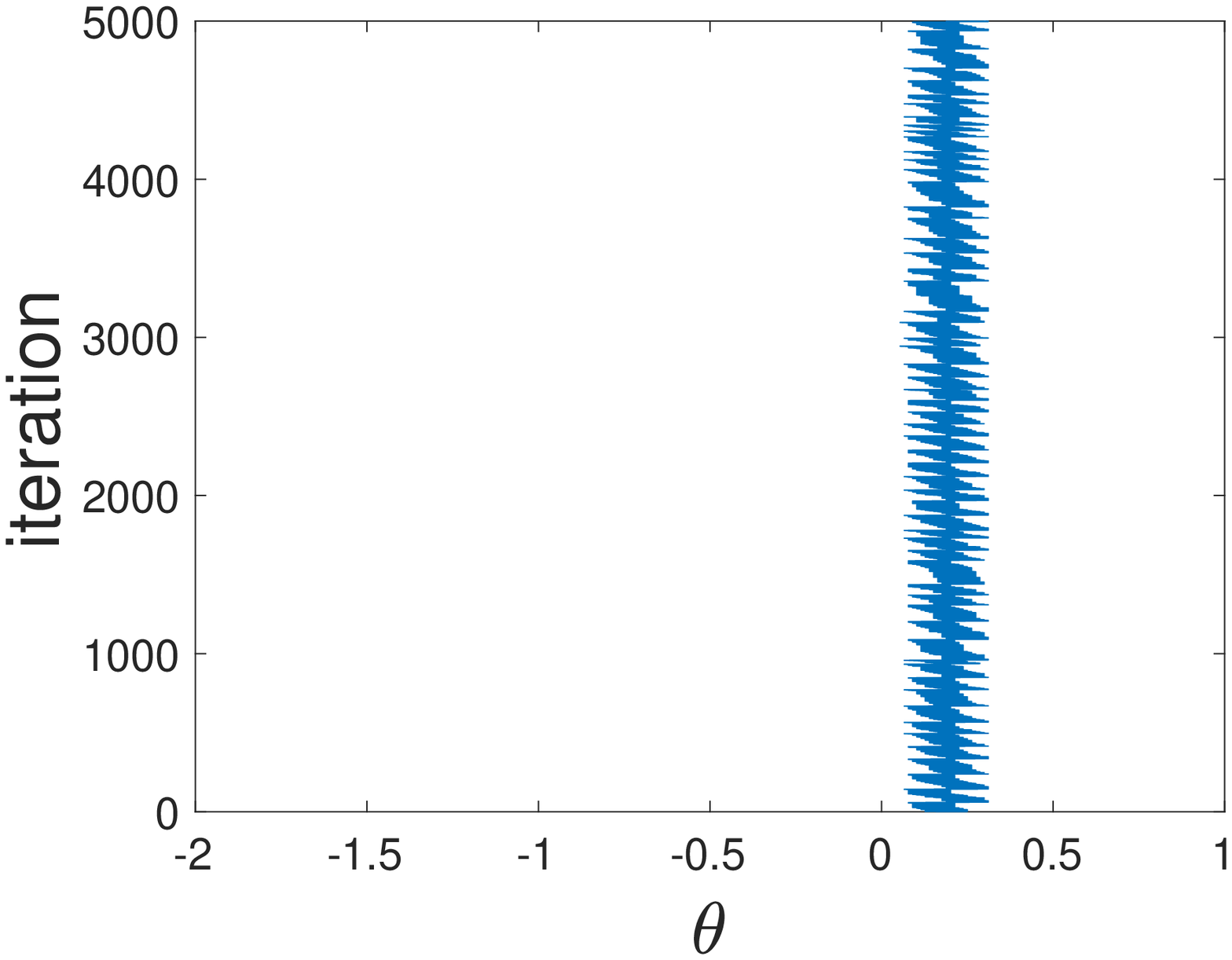}
  \caption{$\eta$ = 0.36}
\end{subfigure}
    \caption{Extended comparisons of SGD (upper row) and RR (lower row) with $a_1 = 0.4, a_2 = 0.6, \epsilon = 0.1, K=5$ at various learning rates $\eta$. All experiments start at $\theta_0 = 0.25$. We can see that unless the learning rate $\eta\leq 0.36$, RR is more reliable in the sense that its trajectory stays around the desired minimum. Even when SGD and RR both stay around the sharp global minimum, the oscillation in RR is smaller.}
\label{fig:moreimages1}
\end{figure}

\subsection{Parameters setup in Section \ref{chemistry}}
\paragraph{M\"{u}ller-Brown potential}
\begin{equation}
    \begin{aligned}
        &V(x,y) = \sum_{i=1}^4 A_i\exp\left(a_i(x-x_i)^2+b_i(x-x_i)(y-y_i)+c_i(y-y_i)^2\right),\\
        & A= (-150, -100, -170, 15),\quad a =(-1, -1, -6.5, 0.7),\\
        & b = (0, 0, 11, 0.6),\quad c = (-10, -10, -6.5, 0.7),\\
        & x = (1, 0, -0.5, -1), \quad y= (0, 0.5, 1.5, 1).
    \end{aligned}
\end{equation}
\paragraph{A large system}
\begin{equation}
    \begin{aligned}
&x_k = 2.0-0.006k, ~~\text{for}~ 1\leq k\leq 500;\quad  x_k = 1.8+0.0024k, ~~\text{for}~ 500\leq k\leq 1000.\\
&y_k = 2.0-0.006k, ~~\text{for}~  1\leq k\leq 500; \quad y_k = -1.0+0.006k, ~~\text{for}~  500\leq k\leq 1000.\\
&A_k = -50.0-0.15k, ~a_k= -2.0-0.018k, \\
&b_k = -0.1+0.0002k, ~c_k=-10+0.009k,~~\text{for}~ 1\leq k\leq 1000.
  \end{aligned}
\end{equation}
\bibliographystyle{plain}
\bibliography{ref.bib}


\end{document}